\documentclass[11pt, letterpaper]{cmu}

\usepackage[all]{hypcap}
\usepackage[comma,numbers,sort,compress]{natbib}
\usepackage{hyperref}[citecolor=magenta]
\usepackage{amsmath, amssymb}

\hypersetup{
    colorlinks = true,
    citecolor = {magenta},
}

    \PassOptionsToPackage{numbers, compress}{natbib}

\usepackage[utf8]{inputenc} %
\usepackage[T1]{fontenc}    %
\usepackage{url}            %
\usepackage{booktabs}       %
\usepackage{amsfonts}       %
\usepackage{nicefrac}       %
\usepackage{microtype}     %
\usepackage{xcolor}         %
\microtypecontext{spacing=nonfrench}
\usepackage{amsmath,amsthm}
\usepackage{graphicx}

\usepackage{mathtools}
\usepackage{mathrsfs}
\usepackage{dsfont}
\usepackage{float}

\usepackage{amssymb}%
\usepackage{pifont}%
\newcommand{\cmark}{\text{\ding{51}}}
\newcommand{\xmark}{\text{\ding{55}}}

\usepackage{xspace}
\usepackage[capitalize,noabbrev]{cleveref}
\usepackage{lipsum}
\usepackage{listings}

\usepackage{bbm}
\usepackage{tabularx}
\usepackage{lmodern}
\usepackage[T1]{fontenc}
\usepackage{lmodern}
\usepackage{amsmath, amssymb}
\usepackage{multirow}
\usepackage[export]{adjustbox}
\usepackage{colortbl}

\usepackage{algpseudocode}
\usepackage{setspace}

\usepackage{color}
\definecolor{deepblue}{rgb}{0,0,0.5}
\definecolor{deepred}{rgb}{0.6,0,0}
\definecolor{deepgreen}{rgb}{0,0.5,0}
\definecolor{boost_correct_to_correct}{HTML}{66C2A5}
\definecolor{default_correct_to_correct}{HTML}{fc8d62}
\definecolor{dup_correct_to_correct}{HTML}{8da0cb}
\definecolor{new_correct_to_correct}{HTML}{e78ac3}

\usepackage[most,skins,theorems]{tcolorbox}

\theoremstyle{plain}
\newtheorem{theorem}{Theorem}[section]

\newtheorem{lemma}[theorem]{Lemma}

\theoremstyle{definition}

\newtheorem{definition}[theorem]{Definition}
\theoremstyle{remark}

\usepackage{amsmath,amsfonts,bm,scalefnt}

\newcommand{\Btr}{B_{\mathrm{tr}}}

\definecolor{lightblue}{rgb}{0.22,0.50,0.70}
\newcommand{\termbold}[1]{{\color{lightblue} \emph{\textbf{#1}}}}

\newcommand{\envfont}[1]{{\scalefont{0.9}{\textsc{#1}}}}

\newcommand{\cdown}{\envfont{Cdown}}
\newcommand{\mathdata}{\envfont{DMath}}
\newcommand{\mult}{\envfont{Mult}}
\newcommand{\multv}{\envfont{Mult-V}}
\newcommand{\Grpo}{GRPO{}}
\newcommand{\GrpoM}{GRPOMask{}}

\def\eqref#1{Eq.~\ref{#1}}

\def\1{\bm{1}}

\DeclareMathAlphabet{\mathsfit}{\encodingdefault}{\sfdefault}{m}{sl}
\SetMathAlphabet{\mathsfit}{bold}{\encodingdefault}{\sfdefault}{bx}{n}

\def\gA{{\mathcal{A}}}

\def\gO{{\mathcal{O}}}

\def\gS{{\mathcal{S}}}
\def\gT{{\mathcal{T}}}

\newcommand{\E}{\mathbb{E}}

\newcommand{\R}{\mathbb{R}}

\DeclareMathOperator*{\argmax}{arg\,max}
\DeclareMathOperator*{\argmin}{arg\,min}

\newcommand{\eqdef}{\mathrel{\mathop=}:}
\newcommand{\paren}[1]{\left( #1 \right)}

\newcommand{\brck}[1]{\left [ #1 \right ] }

\newcommand{\pibase}{\pi_b}

\newcommand{\bs}{\mathbf{s}}
\newcommand{\by}{\mathbf{y}}
\newcommand{\bx}{\mathbf{x}}
\newcommand{\gsim}{\raisebox{-0.13cm}{~\shortstack{$>$ \\[-0.07cm] $\sim$}}~}

\usepackage{wrapfig}
\usepackage{subcaption}

\definecolor{rliableolive}{HTML}{BBCC33}
\definecolor{rliableblue}{HTML}{77AADD}
\definecolor{rliablered}{HTML}{EE8866}

\tcbset{
  aibox/.style={
    width=\linewidth,
    top=6pt,
    bottom=0pt,
    left=1.5pt,
    right=1.5pt,
    colback=rliableolive!13!white,
    colframe=black,
    colbacktitle=black,
    enhanced,
    center,
    attach boxed title to top left={yshift=-0.1in,xshift=0.15in},
    boxed title style={boxrule=0pt,colframe=white,},
  }
}
\newtcolorbox{AIbox}[2][]{aibox,title=#2,#1}
\usepackage{enumitem}

\newcommand\pythonstyle{\lstset{
basicstyle=\ttfamily\footnotesize,
language=Python,
morekeywords={self, clip, exp, mse_loss, uniform_sample, concatenate, logsumexp},              
keywordstyle=\color{deepblue},
emph={MyClass,__init__},          
emphstyle=\color{deepred},   
stringstyle=\color{deepgreen},
frame=single,                       
showstringspaces=false
}}

\lstnewenvironment{python}[1][]
{
\pythonstyle
\lstset{#1}
}
{}

\newcommand\pythoninline[1]{{\pythonstyle\lstinline!#1!}}

\definecolor{blanchedalmond}{rgb}{1.0, 0.92, 0.8}
\definecolor{carmine}{rgb}{0.59, 0.0, 0.09}
\definecolor{lightblue}{rgb}{0.22,0.45,0.70}%

\renewcommand{\mathbf}{\boldsymbol}

\makeatletter
\def\Ddots{\mathinner{\mkern1mu\raise\p@
\vbox{\kern7\p@\hbox{.}}\mkern2mu
\raise4\p@\hbox{.}\mkern2mu\raise7\p@\hbox{.}\mkern1mu}}
\makeatother

\numberwithin{equation}{section}

\definecolor{amaranth}{rgb}{0.9, 0.17, 0.31}
\definecolor{antiquebrass}{rgb}{0.8, 0.58, 0.46}
\definecolor{antiquefuchsia}{rgb}{0.57, 0.36, 0.51}
\definecolor{chromeyellow}{rgb}{0.31, 0.47, 0.26}

\usepackage[dvipsnames]{xcolor}
\definecolor{maj5}{HTML}{2b8cbe}
\definecolor{maj5Imp}{HTML}{084081}

\definecolor{seq5wo}{HTML}{d95f0e}
\definecolor{seq5woImp}{HTML}{662506}

\definecolor{seq5w}{HTML}{6a51a3}
\definecolor{seq5wImp}{HTML}{3f007d}

\definecolor{selfwo}{HTML}{d95f0e}
\definecolor{selfwoImp}{HTML}{662506}

\definecolor{selfw}{HTML}{6a51a3}
\definecolor{selfwImp}{HTML}{3f007d}

\definecolor{glorewo}{HTML}{d95f0e}
\definecolor{glorewoImp}{HTML}{662506}

\definecolor{glorew}{HTML}{6a51a3}
\definecolor{glorewImp}{HTML}{3f007d}

\definecolor{vstar}{HTML}{d95f0e}
\definecolor{vstarImp}{HTML}{662506}

\makeatletter
\def\mathcolor#1#{\@mathcolor{#1}}
\def\@mathcolor#1#2#3{%
  \protect\leavevmode
  \begingroup
    \color#1{#2}#3%
  \endgroup
}
\makeatother

\usepackage[textsize=tiny]{todonotes}
\usepackage{algorithm}
\usepackage{amssymb}

\usepackage[skins,theorems]{tcolorbox}
\Crefformat{equation}{#2Eq.\;(#1)#3}

\Crefformat{figure}{#2Figure #1#3}
\Crefformat{assumption}{#2Assumption #1#3}
\Crefname{assumption}{Assumption}{Assumptions}

\usepackage{crossreftools}
\usepackage[suppress]{color-edits}

\usepackage{dsfont}
\usepackage{nicefrac}

\newtcolorbox{analysisbox}[1][]{
    enhanced jigsaw,
    colback=white,
    colframe=blue!75!black,
    fonttitle=\bfseries,
    boxsep=5pt,
    left=5pt,
    right=5pt,
    top=5pt,
    bottom=5pt,
    title=#1,
}

\tcbset{
  aibox/.style={
    width=\linewidth,
    top=8pt,
    bottom=4pt,
    colback=blue!6!white,
    colframe=black,
    colbacktitle=black,
    enhanced,
    center,
    attach boxed title to top left={yshift=-0.1in,xshift=0.15in},
    boxed title style={boxrule=0pt,colframe=white,},
  }
}
\definecolor{lightblue}{rgb}{0.22,0.45,0.70}%

\usepackage{pifont}
\usepackage{soul}
\definecolor{highlightmistake}{RGB}{255, 179, 179}
\definecolor{highlightcorrect}{RGB}{179, 255, 179}

\title{\texttt{e3}: Learning to \underline{E}xplore \underline{E}nables \underline{E}xtrapolation of Test-Time Compute for LLMs}

\author[1,*]{Amrith Setlur}
\author[1,*]{Matthew Y. R. Yang}
\author[2]{Charlie Snell}
\author[3]{Jeremy Greer}
\author[1]{Ian Wu}
\author[1]{Virginia Smith}
\author[1]{Max Simchowitz}
\author[1]{Aviral Kumar}

\affil[1]{Carnegie Mellon University}
\affil[2]{UC Berkeley}
\affil[3]{Oumi}
\affil[*]{Equal contribution.}

\correspondingauthor{asetlur@andrew.cmu.edu; This work was done at CMU.\\ Project website (with details and a short summary of the key takeaways): \href{https://matthewyryang.github.io/e3/}{https://matthewyryang.github.io/e3/}.}

\newcommand{\ethree}{\texttt{e3}}

\begin{abstract}
\end{abstract}

\begin{document}

\maketitle

\begin{figure}[htbp]
\vspace{-1.0cm}
    \centering
    \begin{subfigure}[b]{0.72\textwidth}
     \captionsetup{font=small,skip=2pt}
        \centering
        \includegraphics[width=0.99\textwidth]{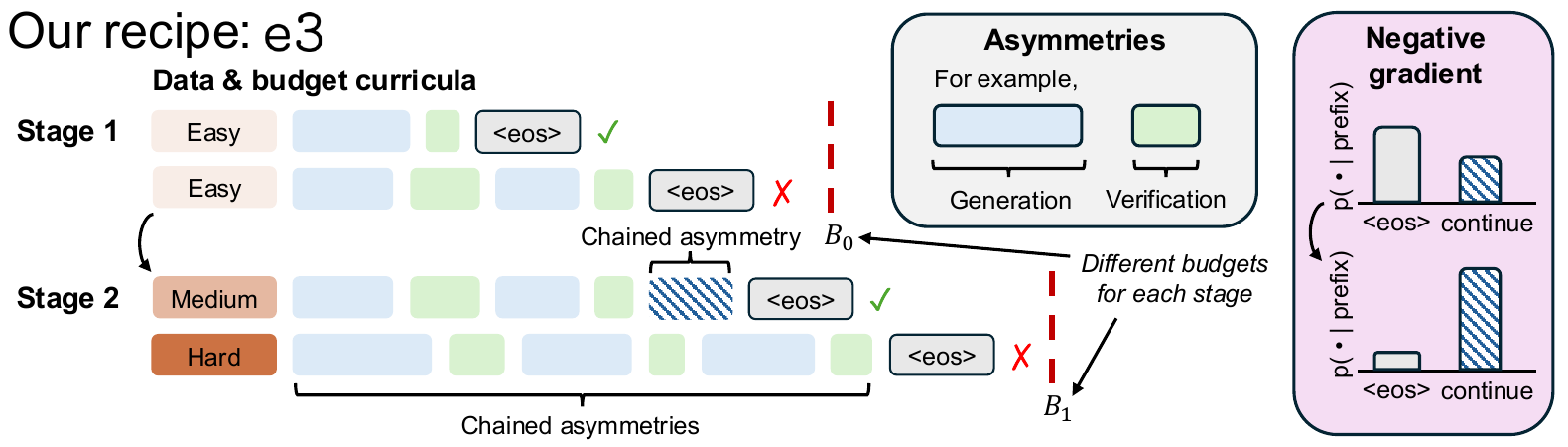}
        \caption*{\centering(a)}
    \end{subfigure}\hfill
    \begin{subfigure}[b]{0.27\textwidth}
     \captionsetup{font=small,skip=2pt}
        \centering
        \includegraphics[width=0.99\textwidth]{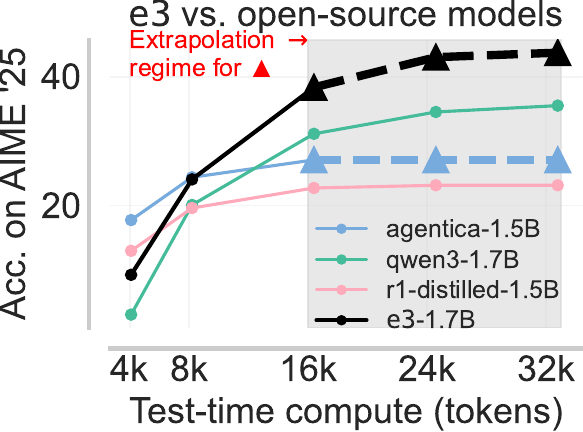}
        \caption*{\centering(b)}
    \end{subfigure}
    \vspace{-0.34cm}
    \caption{\footnotesize{\textbf{\emph{In-context \underline{e}xploration \underline{e}nables \underline{e}xtrapolation
    of test-time compute} (\texttt{e3})}: \textbf{(a)} By (i) chaining asymmetric capabilities of the base model, \textit{e.g.,} reliably self-verifying responses after generating them; (ii) lengthening model responses by chaining more asymmetries until the correct answer is discovered by utilizing the ``negative'' part of the RL policy gradient generated from incorrect responses; and (iii) coupling data \& budget curricula for RL training that carefully structures exploration by sequentially training models on different datasets and training compute budgets. \textbf{(b)} Qwen3-1.7B fine-tuned with \ethree{} outperforms <2B models on AIME' and HMMT' 25 and even some larger 7B/32B models (see full results in Tab.~\ref{tab:passk} and Fig.~\ref{fig:curriculum-panel-1}).}}
    \label{fig:fig1}
    \vspace{-0.2cm}
\end{figure}

{\absfont \textbf{Abstract:} Test-time scaling offers a promising path to improve LLM reasoning by utilizing more compute at inference time; however, the true promise of this paradigm lies in \textit{extrapolation} (\textit{i.e.}, improvement in performance on hard problems as LLMs keep ``thinking'' for longer, beyond the maximum token budget they were trained on). Surprisingly, we find that most existing reasoning models do not extrapolate well. 
We show that one way to enable extrapolation is by training the LLM to perform \emph{in-context exploration}: training the LLM to effectively spend its test time budget by chaining operations (such as generation, verification, refinement, \textit{etc.}), or testing multiple hypotheses before it commits to an answer.
To enable in-context exploration, we identify three key ingredients as part of our recipe \ethree: \textbf{(1)} chaining skills that the base LLM has asymmetric competence in, \textit{e.g.}, chaining verification (easy) with generation (hard), as a way to implement in-context search; \textbf{(2)} leveraging ``negative'' gradients from incorrect traces to amplify exploration during RL, resulting in longer search traces that chains additional asymmetries; and \textbf{(3)} coupling task difficulty with training token budget during training via a specifically-designed curriculum to structure in-context exploration. Our recipe \ethree~produces the best known 1.7B model according to AIME'25 and HMMT'25  scores, and extrapolates to 2$\times$ the training token budget. Our \ethree-1.7B model not only attains high pass@1 scores, but also improves pass@k over the base model.}

\vspace{-0.2cm}
\section{Introduction}
\label{sec:introduction}
\vspace{-0.1cm}

Test-time scaling boosts large language model (LLM) performance by extending inference, spending more compute on ``thinking'' before producing an answer. Its promise lies in enabling models to continue improving performance by scaling test-time compute upon deployment.
\textit{E.g.}, if the model can learn to implement ``algorithmic procedures'' like planning, self-reflection, or backtracking generally across the board, it can discover more accurate responses as more test compute is used. 
With this motivation, current recipes post-train LLMs via reinforcement learning (RL)~\citep{deepseekai2025deepseekr1incentivizingreasoningcapability,yu2025dapo} and supervised fine-tuning (SFT)~\citep{openthoughts,muennighoff2025s1simpletesttimescaling} at long output lengths. However, it is unclear whether the models post-trained with current recipes can truly realize the promise of \termbold{extrapolation}: if we scale the test compute beyond the maximum \emph{training budget}, would the LLM be able to continue to solve more and more problems?

Of course, the performance of a model at very long response lengths may be restricted by other factors like model architecture or model size~\cite{li2024long}. However, one can at least expect that an LLM should benefit from test-time scaling within the pretraining context lengths, that tend to be around 2-4$\times$ larger than the budgets used for training reasoning models\footnote{LLMs often undergo long-context training at the end of pre-trained to 128k tokens (and many proprietary LLMs utilize a million tokens), but during post-training the output length is often reduced to 32k tokens, for instance, for Qwen3 models~\citep{yang2025qwen3}.}. Mechanistically, this could be realized if the LLM were implementing algorithmic procedures (\textit{e.g.,} generate-verify-revise, best-of-$N$, search, etc.) within the model's chain of thought~\citep{kumar2024training,setlur2025opt,gandhi2024stream,yao2023tree}.
However, similar to other empirical studies of reasoning models, we note that many open models perform poorly when extrapolating to 2-3$\times$ the training budget~\citep{qu2025optimizing,hochlehnert2025sober} (Fig.~\ref{fig:aime-line}).
Thus, relying on current RL/SFT recipes to yield extrapolation appears to be mostly futile.

In this paper, we show that the key to enabling extrapolation is \termbold{{learning to explore in-context}}: if a model learns to use compute by searching through multiple reasoning paths or implementing procedures, it can ``guide'' the search towards the correct answer, and improve its performance as more test compute becomes available. Even under the original training compute budget, we expect learning to explore in-context to improve generalization performance to unseen, out-of-distribution problems~\citep{ghosh2021gen,duan2016rl}.
To demonstrate this, we build a recipe \ethree, which trains models that leverage test compute for in-context exploration and can perform well at both normal training and extrapolation budgets. At its core, \ethree{} is based on the following three ingredients and principles (see Fig.~\ref{fig:fig1}):

\textbf{1) Asymmetries are critical for learning to explore.} LLMs can learn to explore only when each segment in the output trace is useful in ``guiding'' subsequent ones, \textit{e.g.}, if verifying initial segments can lead to more refined answers that are more likely to succeed. In the absence of external tools, we show that 
this sort of behavior can emerge from \termbold{asymmetries}, i.e., differences in the model's competence at different skills appearing in an output trace. In the context of self-verification, this corresponds to a verification-generation (VG) gap, where models are more capable of verifying their answers than generating correct ones. 
While prior work~\cite{setlur2025scaling,swamy2025all,song2024mind,kim2025scaling,gandhi2025cognitivebehaviorsenableselfimproving} observed such asymmetries, we formalize their role and show they are essential for enabling RL to increase response length by learning to explore in-context and, as a result, attain extrapolation. Without them, test-time scaling is strikingly hard.
We formalize this notion in a didactic model we call ``$p^k$''-model in Section~\ref{sec:pk_model}: a model of long chain-of-thought LLM training which operates on a base LLM that exhibits perfect self-verification but imperfect generation capabilities. We will show that this asymmetry is critical for enabling extrapolation in this model (discussed next).

\textbf{2) Negative gradient in RL amplifies in-context exploration.} If asymmetries are a prerequisite for learning to explore, what enables them to evolve and facilitate learning useful exploration strategies during post-training? We show that \termbold{negative gradients}~\citep{tajwar2024preference}  (\textit{i.e.}, gradients on incorrect traces) in RL training is a key enabler of in-context exploration when the base model presents asymmetries. Negative gradients drive exploration 
by moving the probability mass from shorter failed traces onto longer traces that \emph{``chain''} new asymmetries (\textit{e.g.}, LLM verifying a calculation one more time). In contrast, SFT only maximizes likelihood on correct traces in the training data and reinforces the model to end the solution within the length of these traces. In our $p^k$ model, SFT only aims to reduce the failure probability $p$ at a fixed $k$, whereas negative gradients also amplify $k$ and increase response length. 

\textbf{3) Structured exploration with coupled curriculum.} 
Finally, while negative gradients amplify asymmetries and produce longer responses, running RL training at very long budgets suffers from poor training convergence, typically seen in long-horizon RL~\cite{agarwal2021theory}. 
Although one could resolve this by training with a smaller budget, we show that training on hard problems at short context lengths often disincentivizes exploration altogether, since the model is forced to commit to an answer prematurely. As a result, we see poor extrapolation of compute and generalization to unseen problems. To resolve this, we design a \termbold{coupled curriculum} over pairs of $(\text{data mixture}, \text{training budget})$ that effectively structures the exploration driven by the negative gradient. 
\emph{\textbf{Our key insight}} is that at any stage of the curriculum, we should choose the smallest ``RL optimization friendly'' budget such that the model initialized for RL training can: (i) complete most of its responses within the budget; and (ii) can continue to improve performance as it chains more asymmetries beyond the chosen budget. 

The above principles and insights constitute our recipe \ethree, that we use to post-train the Qwen3-1.7B model with a training budget of up to 16k output tokens using problems from the DeepScaleR~\citep{deepscaler2025} dataset. We achieve the \textbf{\emph{best performance at $<$2B scale on AIME'25 and HMMT'25}} (to our knowledge), and our model consistently improves as we extrapolate the test-time compute to 32k (2$\times$ the training budget) upon deployment. Our models also attain consistent improvements under the \texttt{pass@32} metric on both of these benchmarks, showing that \ethree{} does more than simply sharpening the base model.

\vspace{-0.4cm}
\section{Problem Statement: Optimizing \& Extrapolating Test-Time Compute}
\label{sec:test-extrapolation}
\vspace{-0.1cm}

\newcommand{\btr}{B_\mathrm{tr}}

\begin{wrapfigure}{r}{0.4\textwidth}
\vspace{-0.8cm}
  \begin{center}
    \includegraphics[trim=0 0 0 0, width=0.99\linewidth]{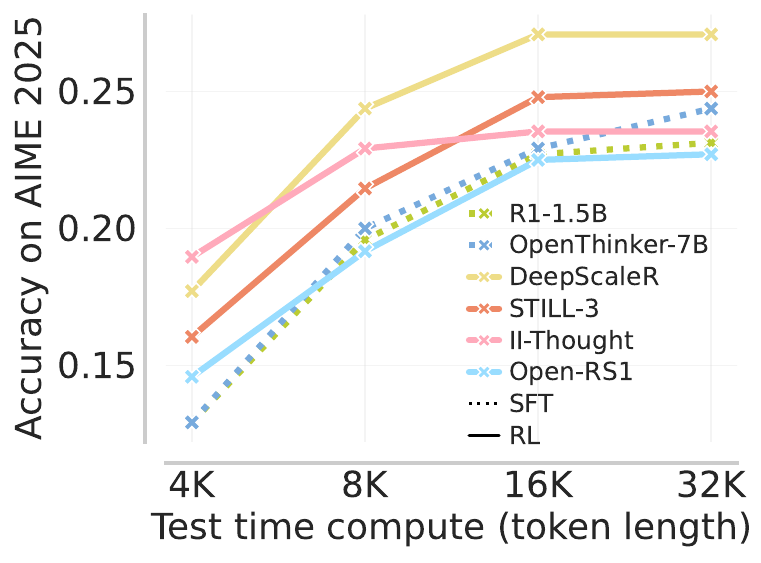}
  \end{center}
  \vspace{-0.6cm}
  \caption{\footnotesize{\textbf{\emph{Accuracy of various open-source models at different budgets on AIME 2025.}} Performance gains diminish as the test-time budget increases, with virtually no gains from 16k to 32k. \vspace{-0.2cm}}}
    \label{fig:aime-line}
    \vspace{-0.4cm}
\end{wrapfigure}
\textbf{Post-training for test-time scaling.} 
RL and SFT are two categories of post-training algorithms that refine a pre-trained base LLM $\pi_b$ into a reasoning model, especially one that utilizes more test-time compute by producing long CoTs.
Typical outcome-reward RL trains LLM $\pi$ (initialized with $\pi_b$) to maximize performance on outcome 0/1 reward  $r^\star(\bx, \by)$, for inputs $\bx \sim \rho$ and response $\by \sim \pi(\by \mid \bx)$ restricted to an apriori fixed maximum token length or \emph{training budget} $B_\mathrm{tr}$~\cite{yu2025dapo,deepscaler2025}. 
On the other hand,  SFT fine-tunes $\pi_b$ on long thinking traces from more capable models or humans to distill their reasoning capabilities~\cite{openthoughts,muennighoff2025s1simpletesttimescaling}, where the maximum length of the expert traces also implicitly induces a training budget $B_\mathrm{tr}$, similar to RL. 
\textbf{Our goal} is to train models that can improve performance when we extrapolate test-compute beyond $B_\mathrm{tr}$. 
Even though the true promise of test-time compute is extrapolation performance, we find that \emph{\textbf{current thinking models fall short on extrapolation.}} We evaluate multiple models on a test budget of 32K,  $\approx$1.5-2$\times\btr$ across all models. We plot our results on AIME25 in Fig.~\ref{fig:aime-line} (see App.~\ref{sec:test-extrapolation-appendix} for a detailed comparison) and note that performance gains are minuscule as we go beyond $\btr$.

\textbf{Negative gradient in RL.} A key distinction between SFT and RL is the \termbold{negative gradient}, which corresponds to the part of the policy gradient coming from traces that fail.  In \eqref{eq:policy-grad} we present a simplified yet generalizing version of the policy gradient adopted by most RL post-training methods: REINFORCE~\cite{ahmadian2024back}, PPO~\cite{schulman2017proximal}, and GRPO~\cite{shao2024deepseekmath}. 
From this, we note that on a prompt $\bx$, RL training observes two types of gradients: (i) the positive gradient which maximizes the likelihood of a correct responses $\by$ with a positive advantage $A(\bx, \by)$, and (ii) the negative gradient which \emph{pushes down} the likelihood of an incorrect response with a negative advantage $A(\bx, \by)$. Here, $\by$ can be sampled \emph{on-policy} $\pi=\tilde{\pi}$ or \emph{off-policy} $\pi\neq \tilde{\pi}$. Thus, we can view  SFT as a purely positive gradient method that only maximizes likelihood on correct reasoning traces. 
In Sec.~\ref{sec:negative-gradient}, we show why the negative gradient is largely responsible for driving up response lengths and in-context exploration during RL, thereby enabling RL-trained models to explore more at test-time and extrapolate better compared to SFT-based ones. 
{
\setlength{\abovedisplayskip}{5pt}
\setlength{\belowdisplayskip}{3pt}
\begin{align}
    \label{eq:policy-grad}
    \mathbb{E}_{\mathbf{y} \sim \tilde{\pi}(\cdot \mid \mathbf{x})} \;  \brck{A_i(\mathbf{x}, \mathbf{y}) \cdot \nabla_\pi\log \pi (\mathbf{y} \mid \mathbf{x})}  \quad \text{{(simple form of policy gradient in RL)}} 
\end{align}
}

\newcommand{\bte}{B_\mathrm{ex}}

\vspace{-0.2cm}
\section{Asymmetries in the Base Model: A Prerequisite for In-Context Exploration}
\label{sec:vg-gap}
\vspace{-0.1cm}

\emph{How can extrapolating beyond the training budget improve performance?} To answer this, we begin by revisiting why longer traces perform better in general. The conventional wisdom is that longer traces can represent solutions that make multiple attempts, interleaving verification and generation~\citep{setlur2025scaling,nie2024evolve,krishnamurthy2024can}, to arrive at the final answer. We can think of this as the LLM learning to interleave basic ``skills'', \textit{e.g.}, verification, summarization, or retrieval, to perform {in-context exploration}. But why, or when,  should a post-training recipe favor learning such solution traces over other uses of test-time compute that arrive at an answer more directly? This section demonstrates that when the base model exhibits \textit{asymmetric} incompetence at  different skills, 
RL post-training prefers to learn solutions that  \termbold{chains asymmetric skills} in ways that improve final performance. A formal description is given by the following definition:

 \newcommand{\detect}{\mathtt{detect}}
 \begin{tcolorbox}[colback=black!6!white,colframe=black,boxsep=0pt,top=4pt,bottom=4pt,left=3pt,right=3pt]
\begin{definition}[Chaining asymmetric capabilities $p, q$ in model $\pi$.]
\label{def:asymmetry}
    Let $p, q:\gS \mapsto \gS$ be functions over token sequences $\gS$ (\textit{e.g.}, $p$ can be generation, $q$ can be verification), and $\detect(f, \tau)$ detects number of calls to function $f$ in a token trace $\tau$. For a reward $r$, we say that policy $\pi$ chains asymmetries $p, q$ if it benefits from calls to the composition $q(p(\cdot))$, compared to only $p(\cdot)$:
    {
    \setlength{\abovedisplayskip}{3pt}
    \setlength{\belowdisplayskip}{3pt}
    \begin{align*}    
    \E_{\tau \sim \pi}\brck{r(\tau) \mid \detect(q(p(\cdot)), \tau)> 0} \; > \; \E_{\tau \sim \pi}\brck{r(\tau) \mid \detect(p, \tau) > 0}, 
    \end{align*}}even though there is an optimal policy $\pi_r^\star$ that never calls $q$, i.e., $\E_{\tau \sim \pi_{r}^\star}\brck{\detect(q, \tau)}=0$.
\end{definition}
\end{tcolorbox}
We focus on a key special case when the model is more accurate at verifying its own answers than it is at generating correct ones; that is, when the model exhibits a \emph{verification-generation gap} (\termbold{VG Gap}), 
on a particular problem domain~\cite{song2024mind,setlur2025scaling,swamy2025all}. In this section, we show that RL training on problem domains with VG gap (i) encourages chaining asymmetries, (ii) enables in-context exploration that (iii) discovers new solutions, often extrapolating to larger budgets and more difficult problem domains. 

\newcommand{\ToyBase}{Llama3.2-3B}
\begin{wrapfigure}{r}{0.36\textwidth}
\vspace{-0.6cm}
  \begin{center}
    \includegraphics[width=0.99\linewidth]{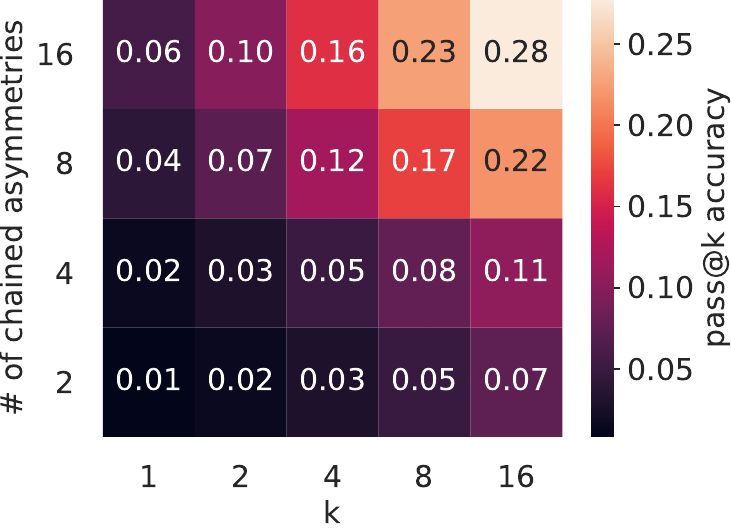}
  \end{center}
  \vspace{-0.6cm}
  \caption{\footnotesize{\textbf{\emph{Measuring asymmetry (Def.~\ref{def:asymmetry}) \& \texttt{pass@k} on \cdown.}} Pass@k improves more for all $k$ as the number of chained asymmetries increases in a trace from  \ToyBase.}}
\vspace{-0.3cm}
    \label{fig:base-asymmetries}
\end{wrapfigure}
\textbf{Setup.} We validate the role of asymmetries in learning to explore by investigating two didactic tasks, on which \ToyBase~ admits different VG gaps. First, the \textbf{Countdown game}~\citep{yao2023tree,gandhi2024stream} (\cdown) requires converting a set of numbers into an equation that evaluates to the desired target. The base LLM is more effective at verifying whether a proposed equation evaluates to the target than searching over all possible equations to solve the task, and traces with more chained asymmetries are more performant, as we measure pass@k in Fig.~\ref{fig:base-asymmetries}, where performance on traces with more chains is higher. 
Second, we study \textbf{n-digit multiplication} (\mult{}) in natural language, without any {external} tools, where the base model exhibits limited verification (see App.~\ref{sec:vg-gap-appendix} for asymmetry gap on \mult). Additionally, we fine-tune  \ToyBase~on correct n-digit multiplication traces from Qwen-32B-r1-distilled. These traces contain multiple verification steps verifying intermediate steps, like smaller digit multiplications, that are part of  a longer trace (see App.~\ref{sec:vg-gap-appendix} for an example). This fine-tuning is a direct way to encourage more verification attempts (\multv{}). Comparison of \mult~vs.~\multv~enables direct evaluation of the benefits of asymmetries in base LLM, all else being held equal. In these results, we detect verification segments by separating by the ``$\backslash$n$\backslash$n'' token (see App.~\ref{sec:example-traces} for examples).

\emph{\textbf{Finding 1: Verification-generation  asymmetry in the base model improves the performance of RL trained solutions.}}
Fig.~\ref{fig:vg-gap-main}(a,b) shows a stark difference in performance and length of output traces as the training budget $\btr$ varies on \cdown~and \mult. 
On \cdown,   performance consistently increases as $\btr$ increases from 512 $\rightarrow$ 2048, accompanied by a very clear increase in response length. On \mult, where the base model has limited propensity to verify,  performance increases when $\btr$ increases from 1024 to 2048, but it plateaus thereon. Unlike \cdown, test-time length is far from saturating budget limits and also oscillates widely across RL training epochs.  Contrast this with Fig.~\ref{fig:vg-gap-main}(d), where RL training on \multv, which leverages verification, exhibits longer lengths and stronger extrapolation performance. Overall, this implies that \textbf{leveraging asymmetries improves performance and length-utilization in RL postraining.} Curiously, we also observe that models with greater VG gap exhibit less KL divergence from the base model, perhaps implying better generalization -- see App.~\ref{sec:vg-gap-appendix} for those results and discussion.  

\begin{figure}
    \centering
    \includegraphics[width=0.99\linewidth]{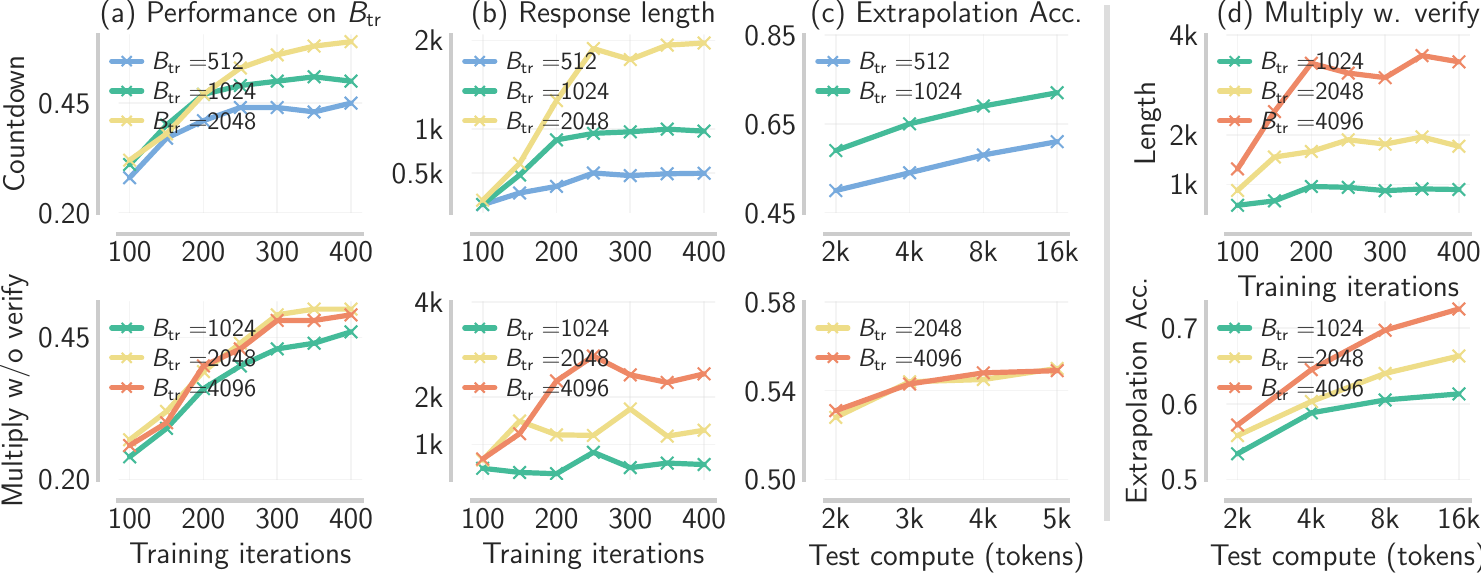}
    \vspace{-0.15cm}
    \caption{\footnotesize{\textbf{\emph{RL training with and without asymmetries in the base model.}} When asymmetries such as the VG gap are present (\textit{e.g.,} in \cdown), RL training amplifies response length by chaining more asymmetries to explore in-context, where the probability of success improves with higher length on both $\btr$ and extrapolation regimes. On the other hand, when VG gap is absent in $\pibase$ (\textit{e.g.,} in \mult), increases in length and extrapolation performance are subdued. When we explicitly train on a base model fine-tuned to verify \mult~ (a setting we refer to as the \multv), we again observe upward length and extrapolation trends, consistent with \cdown.}}
    \label{fig:vg-gap-main}
        \vspace{-0.5cm}
\end{figure}

\emph{\textbf{Finding 2: Chaining asymmetries enable extrapolation via in-context exploration.}} Interleaving verification and generation steps chains together asymmetric skills of the base model; we refer to this special case of skill-chaining as \termbold{chaining asymmetries}. To measure the benefits of chained asymmetries on \cdown, we plot the pass@k accuracy of the base LLM, shown in Figure~\ref{fig:base-asymmetries}, and observe that performance increases as more chained asymmetries arise. In fact, the best strategy is to not simply scale $k$, but rather to scale both $k$ and the number of chained asymmetries (more details on this experiment are in Appendix~\ref{sec:vg-gap-appendix}).
In Fig.~\ref{fig:vg-gap-main} (c), we plot the extrapolation performance of the models trained at two values of $B_\mathrm{tr}$. On \cdown~the model trained with $B_\mathrm{tr}$ 0.5-1k makes steady progress on problems in test budgets that are 8-16$\times$ $B_\mathrm{tr}$ itself. On \mult, we find that $B_\mathrm{tr}$ has absolutely no effect on extrapolation performance with the base LLM that does not have VG asymmetry, but it has a substantial effect when the asymmetry is present. More importantly, while the base model without VG asymmetry fails to extrapolate and solve unsolved problems, with its accuracy improving by merely $\leq 2\%$ despite 16$\times$ test-time compute scaling, the base model with VG asymmetry can still extrapolate well.

\textbf{Why do asymmetries enable in-context exploration?.} We explain this via our didactic {$p^k$-model} (details in Sec.~\ref{sec:pk_model}). Here we view the LLM as sequentially guessing responses $a_1,a_2,\dots$, each with failure probability $p$, and up to at most terminal $k$ responses. We assume that the LLM admits perfect verification, meaning that it can decide when to stop or continue perfectly. Now, in a simplified setting where attempts are independent, failure probability (= $p^k$) decays exponentially as $k$ increases, as on \cdown. Therefore, we can improve performance by increasing $k$ and $p$ together. 
However, if verification is difficult, increasing $k$ provides little benefit, since the model cannot adjudicate whether one guess is any better than another. Then, the only way to improve performance is by lowering $p$ (better first guesses as seen on \mult).

\begin{AIbox}{Takeaways: Asymmetries are a critical pre-requisite for learning to explore.}
\begin{itemize}[leftmargin=0.7em]
    \setlength\itemsep{0em}
    \item  Asymmetries like the VG gap enable the model to continually explore, verify, and refine answers. 
    \item RL training amplifies chaining of asymmetric skills and produces solutions that learn to explore in-context, thus benefiting from additional test-time compute beyond the training budget. 
\end{itemize}
\end{AIbox}

\begin{figure}[t]
    \centering
    \includegraphics[width=0.99\linewidth]{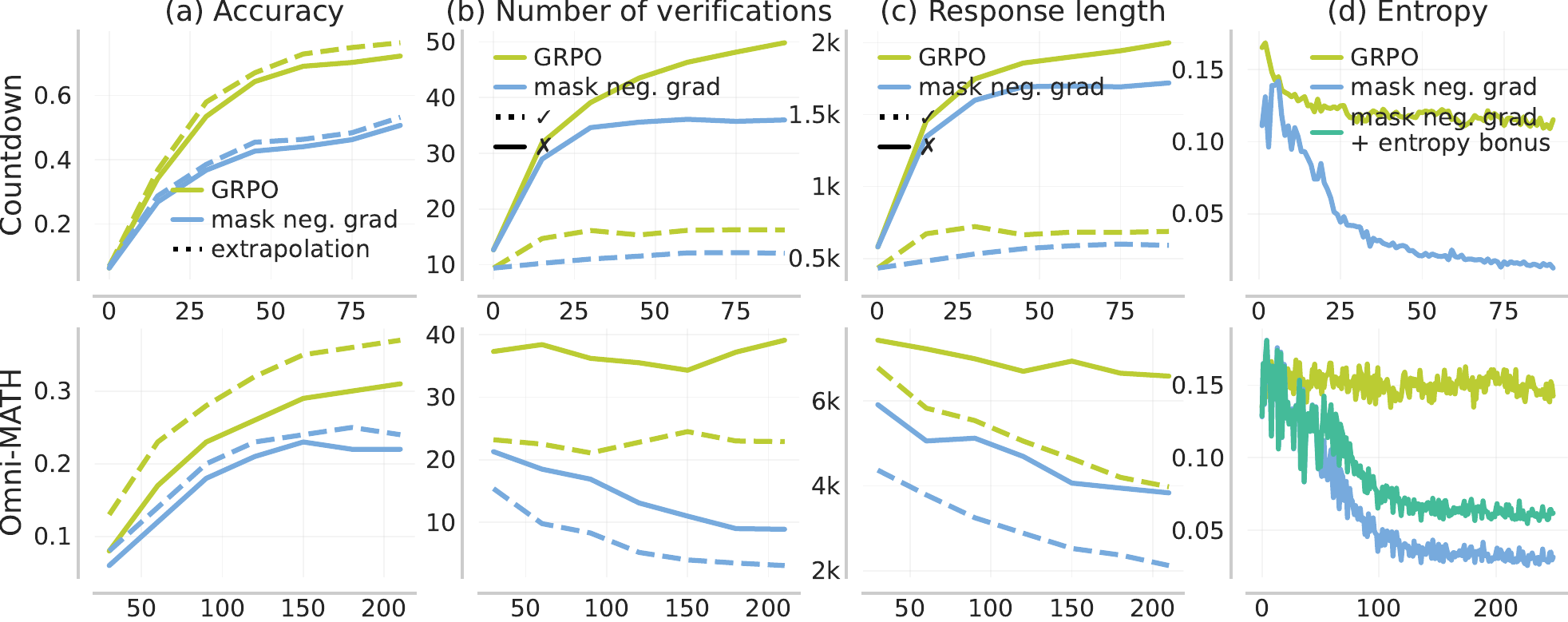}
    \vspace{-0.15cm}
      \caption{\footnotesize \textbf{\emph{RL training with and without negative gradients:}} When the base model admits asymmetries, negative gradients promote in-context exploration by: (i) increasing length ((c)) and chaining asymmetries, which shows up as more verification attempts (b); and (ii) increasing token entropy and thus response diversity (d). This leads to better performance on the training budget and upon extrapolation.  In (b, c), $\cmark$ denotes the statistic computed on correct responses and $\xmark$ on incorrect responses. \vspace{-0.5cm}}
    \label{fig:negative-grad-main}
\end{figure}

\vspace{-0.3cm}
\section{Negative Gradients Incentivize Exploration that Chains Asymmetries}
\label{sec:negative-gradient}
\vspace{-0.2cm}
Having observed that the presence of asymmetry in the base model is a prerequisite for in-context exploration, the next question is: What enables models to exploit and chain these asymmetries during RL?  
In this section, we show that a crucial ingredient here is the \termbold{negative gradient}, the gradient term multiplied by a negative advantage in \eqref{eq:policy-grad}. Negative gradient drives in-context exploration via two mechanisms: (i) incentivizing sampling of unseen token sequences; (ii) chaining asymmetries like VG gap (Sec.~\ref{sec:vg-gap}) that rapidly drives up response length and the number of verification attempts.  Note that while mechanism \textbf{(i)} corresponds to the \textit{classical} notion of exploration, mechanism \textbf{(ii)} is special in that it corresponds to a form of ``structured'' exploration over strategies already in the model. When put in RL terminology, this corresponds to ``meta exploration''~\citep{liu2020learning,gupta2018umrl}. We study these effects in this section.

\begin{wrapfigure}{r}{0.5\textwidth}
\vspace{-0.7cm}
  \begin{center}
    \hspace{-0.3cm}
    \includegraphics[width=0.49\linewidth]{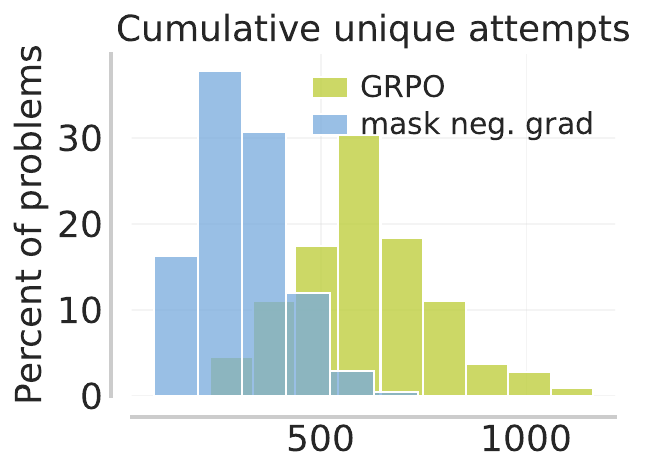}
    \includegraphics[width=0.49\linewidth]{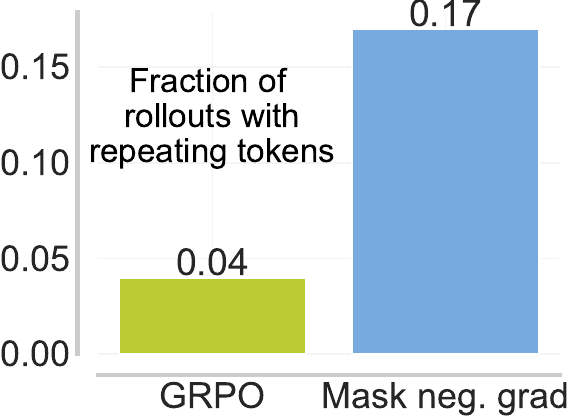}
  \end{center}
  \vspace{-0.6cm}
  \caption{\footnotesize{
   \textbf{\emph{Negative gradient encourages distinct responses:}} it increases the cumulative number of unique attempts on \cdown~(left) and reduce responses that end with a repeating stream of tokens on \mathdata~(right) vs. masking it out.
  \vspace{-0.2cm}}}
\label{fig:unique-attempts}
\vspace{-0.3cm}
\end{wrapfigure}
\textbf{Analysis setup.} We analyze the evolution of response length, performance, and the number of chained asymmetries by comparing two training algorithms: (i) standard outcome-reward RL using \Grpo~\citep{shao2024deepseekmath} with token-level normalization~\citep{yu2025dapo}; (ii) \GrpoM, which zeros out (i.e. masks) the negative gradient and whilst retaining the \emph{positive} gradient,  thereby resembling an approach close to online STaR~\cite{zelikman2022star} or RFT~\cite{yuan2023scaling}. We conduct our experiments on \cdown~and \mathdata~reasoning (questions sourced from  DeepScaleR dataset~\citep{deepscaler2025}) that exhibit the VG asymmetry. We make the following observations:

\emph{\textbf{Finding 1: Negative gradients promote diverse responses during RL training}}, encouraging exploration at two levels: (i) within a rollout; and (ii) across different rollouts. For (i), we observe that removing the negative gradient results in an entropy collapse over the next-token distribution (Fig.~\ref{fig:negative-grad-main} (d)). 
 This curtails diversity and leads to responses with a repeating stream of tokens when extrapolating the trained model to larger budgets (Fig.~\ref{fig:unique-attempts}). For (ii), we measure the cumulative unique attempts on the \cdown~test set as we train the model (Fig.~\ref{fig:unique-attempts}) by separating each rollout into attempts using ``\textbackslash n\textbackslash n'' and parsing the equations from each attempt. An attempt is unique when its equations differ from those of other attempts from rollouts across all gradient steps. We find more unique attempts when training with negative gradients. Therefore, utilizing the negative gradient clearly enhances exploration. While it is not surprising that RL algorithms benefit from exploration~\cite{hazan2019maxent}, we next explain how, distinctly from standard RL, this exploration can be particularly effective when extrapolating to larger test budgets.

\emph{\textbf{Finding 2: Negative gradient increases the number of chained asymmetries}}, and thereby boosts structured exploration (and extrapolation) as we show next. Concretely, when training on an incorrect response $\by$ with tokens $y_1, y_2,.., \texttt{EOS}$, the negative gradient reduces the conditional probability of each token $y_i$ conditioned on the prefix $y_{1:i-1}$ appearing in this response, i.e. $p(y_i|\by_{1:i-1})$. This process also reduces the probability of the \texttt{EOS} token: $p(\texttt{EOS} |\by)$, for any incorrect response that ends within the response budget. Where does this probability mass go? Clearly since total probability is conserved, this probability mass must be repurposed to increase the likelihood of other tokens. Fig.~\ref{fig:negative-grad-main}(b) shows that the probability mass recovered from the negative gradient is repurposed to increase the probability of chaining new pairs of asymmetric skills to the current trace (\textit{e.g.}, ``Wait, ...'' instead of terminating with $\texttt{EOS}$). This chaining results in a greater response length \textbf{(c)} and higher overall performance. 

When negative gradients are masked (\GrpoM) in \cdown, we see that attempts Fig.~\ref{fig:negative-grad-main}(b) and length Fig.~\ref{fig:negative-grad-main}(c) plateau, accompanied by a decrease in performance. The relative trends between \Grpo{} and \GrpoM~are similar for \mathdata, but differ in absolute values (e.g., the number of absolute chained asymmetries decline in the absence of negative gradients). We include further results in App.~\ref{sec:negative-gradient-appendix} (Fig.~\ref{fig:neg-grad-multiply}), where we also demonstrate that \mult~(which does not exhibit asymmetries) benefits far less from negative gradients. This mechanism for boosting exploration by chaining new asymmetries is different from the typical notions of improving coverage or trying novel tokens discussed in Finding 1.

\emph{\textbf{Finding 3: LLMs trained with negative gradients extrapolate better.}}
Finally, we explain why negative gradients enable extrapolation. Longer responses that chain asymmetries are more likely to yield correct answers and thus receive positive reward. Therefore, the policy gradient update reinforces chaining and improves in-context exploration, and this process exhibits a ``rich gets richer'' effect, where further training incentivizes more in-context exploration (since the gap between number of verifications with and without negative gradient increases as training progresses in Fig.~\ref{fig:negative-grad-main}(b)). As discussed in Sec.~\ref{sec:introduction}, models that learn to explore in-context benefit from additional test-time compute—greater search leads to better performance under large value gaps. Fig.~\ref{fig:negative-grad-main}(a) confirms this: on hard \mathdata\ problems (we classify a problem as hard if QwQ-32B attains pass@32 performance of \emph{\textbf{zero}}), doubling the test-time budget amplifies the performance gap when negative gradients are used, compared to the masked variant.

\begin{AIbox}{Takeaways: Negative gradient incentivizes in-context exploration with large VG gaps}
\begin{itemize}[leftmargin=0.7em]
    \setlength\itemsep{0em}
    \item Negative gradients in RL ``move'' probability from short-length incorrect answers onto other modes, \textit{e.g.,} those that exploit asymmetries or those that end in a correct answer. When the VG gap is large, longer responses that chain more asymmetries and eventually discover the right answer are rewarded and reinforced. As a result, in-context exploration is reinforced. 
    \item Negative gradients boost response diversity and thus coverage over correct answers, as confirmed by our empirical results on \cdown{}, \mathdata{}, and theoretical results in the bi-gram model. 
\end{itemize}
\end{AIbox}

\newcommand{\Aspace}{\gA}
\newcommand{\Abarspace}{\bar{\gA}}
\newcommand{\astar}{a^\star}
\newcommand{\stopact}{\bm{\texttt{stop}}}
\newcommand{\Laws}{\triangle}
\newcommand{\pinot}{\pi_{0}}

\vspace{-0.25cm}
\section{Analyzing Negative Gradient Dynamics in the \texorpdfstring{$p^k$}{pk} Model}
\label{sec:pk_model}
\vspace{-0.2cm}

In this section, we introduce a didactic $p^k$ model, where an LLM samples $k$ independent actions sequentially, verifies them (with a perfect accuracy), and terminates immediately after the correct one is produced. In this section, we introduce a didactic setup where verification is perfect (and hence, there is a high VG gap), and formalize the intuitions regarding negative gradient from the previous section.  

\textbf{Didactic analysis setup.} We consider a Markov decision process (MDP)~\citep{puterman1994markov}
with action space $\Abarspace = \Aspace \cup \{\stopact\}$, where $\Aspace = [100]$  are standard actions and $\stopact$ is an early ``stopping'' action (like \texttt{EOS}) that terminates the trace. For simplicity, we consider policies parametrized as a softmax bigram model $\pi_{M}(a_{t+1} \mid a_t)$: in this model, the policy only retains one token in its history and is parameterized by a softmax over logits described by bi-grams, i.e., $\pi_M(a_{t+1} | a_t) \propto \exp(M(a_t, a_{t+1}))$. In this bi-gram model, the current state $s_t$ always matches the previous action $a_{t-1}$, and $a^\star \in \Aspace$ denotes the optimal action. In a rollout $a_1,..., a_t$, the initial action $a_1$ is sampled from a fixed $\pinot$. For $t > 1$, a learner policy samples an action $a_t$$\sim$$\pi(\cdot|a_{1:t-1})$ $\in$ $\Laws(\Aspace)$. The MDP terminates with reward $1$ at time $t$ if $a_t = \astar$,  and with reward $0$ if  $a_t = \stopact$ (stops too early),  or  $t > \btr$ (budget is exhausted before a correct response). The policy is initialized to one that puts a high probability mass on choosing $a=$ \texttt{stop}. Details are in App.~\ref{sec:negative-gradient-appendix}.

We say that the model \textbf{learns to explore in-context} if it learns to never play $\stopact$ for any $t$ (no early stopping), until $a^\star$ is observed, \textit{i.e.}, increasing $k$ in $p^k$.  On the other hand, \textbf{classical exploration} amounts to upweighting $\pi(\astar \mid a_{1:t-1})$ without reducing $p(\stopact)$, \textit{i.e.}, improving $p$ in $p^k$.

\begin{figure}
    \centering
    \includegraphics[width=0.99\linewidth]{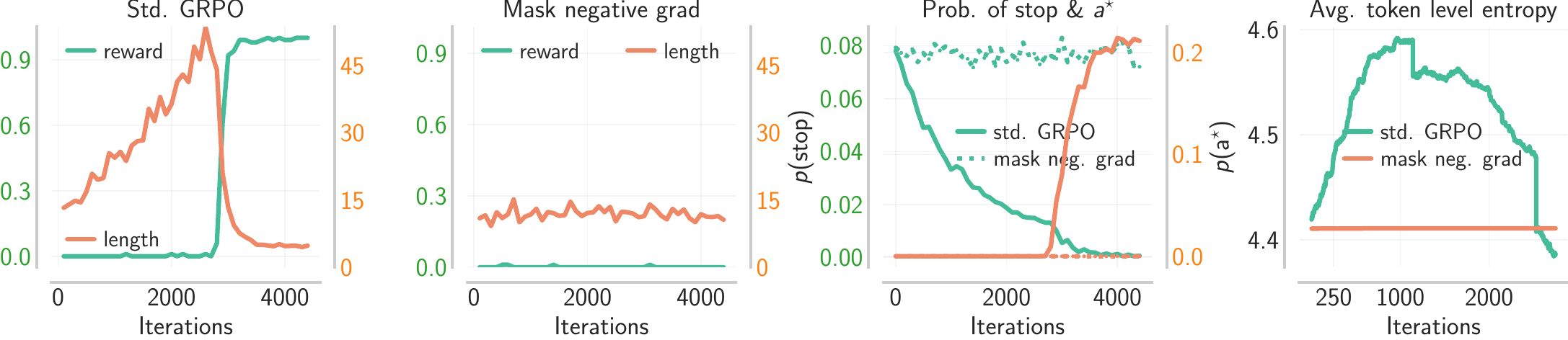}
    \vspace{-0.1cm}
    \caption{\footnotesize{\emph{\textbf{Negative gradients in the $p^k$-model.}}} Negative gradients push down $p(\stopact)$ during training (c), increasing length (a) and entropy of the next action distribution (d) to accommodate more in-context exploration, only decreasing them when $\astar$ is discovered. In contrast, positive gradients rarely change $p(\stopact)$ or entropy.  \vspace{-0.55cm}}
    \label{fig:simulation-main-fig}
\end{figure}

\emph{\textbf{Finding 1: Negative gradient increases length until $p(a^\star)$ is reasonably high.}} In Fig.~\ref{fig:simulation-main-fig}(a), standard \Grpo~($\btr=100$) increases average response length from $15$ to $45$ at budget, driven by the drop in the marginal probability of stopping early $p(\stopact)$ (Fig.~\ref{fig:simulation-main-fig}(c)). After multiple RL iterations with negative gradients, the average number of attempts per trace is sufficiently large, and the learner can sample $a^\star$ with non-trivial probability in any given trace. Once this happens (Fig.~\ref{fig:simulation-main-fig}(c)), in our simple bigram setup, the model rapidly upweights the likelihood of one-step transitions to $a^\star$, resulting in a phase transition where reward increases as length drops. In contrast, \GrpoM~(Fig~\ref{fig:simulation-main-fig}(b)) fails to improve reward or increase length.  The first phase is akin to chaining more asymmetries in LLMs and results  in a longer response length. In our LLM benchmarks, however, we do not see the same phase transition since finding ``shortcuts'' to correct responses is considerably more difficult. Moreover, the LLM is conditioned on an entire history and learns to utilize the history carefully in the first phase. This makes it unlikely for it to quickly learn to reduce length substantially even if it transitions into this second phase on some problems.

\emph{\textbf{Finding 2: Negative gradient improves coverage by increasing entropy of $\pi_M(\cdot \mid a_{1:t-1})$.}} When $\pi_M$ samples a highly likely yet incorrect action, the negative gradient computed on this sample increases entropy by moving probability mass onto less-seen modes of the distribution, including $a^\star$. Note that no explicit entropy bonus is applied. We show this formally in Theorem~\ref{thm:neg-gradient-entropy-informal} where we prove that upon sampling a highly likely incorrect action with probability $p$, \Grpo{} update with a negative gradient results in an entropy increase of $\approx p^2$ when all other actions, including $a^\star$ are highly unlikely. We note this empirically as well in Fig.~\ref{fig:simulation-main-fig}(d), where conditional entropy increases across states, until $a^\star$ is discovered, after which it drops sharply as the positive gradient rapidly moves mass onto $a^\star$ within a few iterations.

\begin{theorem}[Negative gradient increases entropy when $a^\star$ is unlikely;  formal version in Thm.~\ref{thm:neg-gradient-entropy-formal}] 
\label{thm:neg-gradient-entropy-informal}
At state $\bs$, if the most likely action under $\pi$ is $a_1 \eqdef\argmax_{a'}\pi(a' | \bs) \neq a^\star$, 
then, for any $\pi$, a  negative stochastic gradient step increases the entropy of $\pi(\cdot|\bs)$ with prob. $\geq \pi(a_1 | \bs)$. Additionally, in a suitable regime of $\pi$, the  increase $\gsim$$(\pi(a_1| \bs)-\pi(a_2| \bs))^2$, where $a_2$ is second most likely after $a_1$. In contrast, in the absence of the negative gradient, the entropy is preserved with prob. $1-\pi(a^\star | \bs)$.
\end{theorem}

\vspace{-0.4cm}
\section{Coupled Curriculum Training Structures Exploration in Long Length RL}
\label{sec:structured-exploration}
\vspace{-0.1cm}

In the presence of asymmetries, training with negative gradients produces models that can extrapolate beyond their training budget.
However, of course, training on just \emph{any} arbitrarily chosen training token budget $\btr$ is not enough: if $\btr$ is too small, then we would not expect any form of in-context exploration to emerge. Perhaps unsurprisingly it turns out that a much larger $\btr$ is also not sufficient. In Fig.~\ref{fig:curriculum-panel-0}(a), we show that different training budgets $B_\mathrm{tr}$ lead to different levels of performance on the training budget, as well as extrapolated test compute. \emph{So how should we set the budget $B_\mathrm{tr}$ to attain strong extrapolation performance?} And in correspondence with token budgets, \emph{what prompts should we be training on for a given budget?} To answer these questions, we run several training runs at different budgets.

\vspace{-0.3cm}
\subsection{Training on Static Budgets or Data Mixtures is Insufficient}
\vspace{-0.2cm}

\textbf{Setup.} We evaluate extrapolation performance on \mathdata{} and \cdown{} after  training on  different budgets and prompt compositions. We split \mathdata{} evenly across three levels of hardness as measured by the performance of Qwen-R1-Distilled-32B accuracy.
For \cdown, we judge problem difficulty based on the number of terms in the equation. We use the GRPO~\cite{shao2024deepseekmath} algorithm to train models on all compute budgets and datasets (see App.~\ref{sec:structured-exploration-appendix} for the hyperparameter configurations we use).

\textbf{\emph{Finding 1: Training solely at low or high $B_\mathrm{tr}$ is not desirable.}} We train on the easy \mathdata~problems at different training budgets $B_\mathrm{tr}$= 4k, 8k,16k (see Fig.~\ref{fig:curriculum-panel-0}(a)). While training at the short budget $B_\mathrm{tr}$= 4k attains the best performance at the same test budget of 4k tokens, it ``kills'' in-context exploration since traces with many chained asymmetries are typically longer than the training budget of 4k and traces that might succeed after spending 4k tokens are negatively rewarded. Overall, this hinders length increase and chaining of asymmetries driven by the negative gradient, leading to poor extrapolation (no gains from 8k to 32k). 
Fig.~\ref{fig:curriculum-panel-0}(c) shows that this biases the model to stop early and terminate incorrectly.

On the other extreme, training at $B_\mathrm{tr}$= 16k introduces significant optimization challenges,
typical of policy gradients in long horizons suffering from high gradient variance~\cite{agarwal2021theory}. This model performs worse on its own training budget of 16k compared to a model trained on  $\btr$= 8k and extrapolated ``zero-shot'' to 16k. We find that $B_\mathrm{tr}$= 8k attains the best scaling when extrapolating test compute, implying that the choice of $\btr$ needs to strike a balance between: (i) the length budget available for negative gradient to encourage chained asymmetries (infeasible in $<$4k tokens); and (ii) mitigating optimization challenges.%

\begin{figure}[!t]
     \captionsetup{font=small,skip=0pt}
    \centering
    \begin{subfigure}[b]{0.24\textwidth}
    \captionsetup{font=small,skip=2pt}
        \centering
        \includegraphics[width=0.99\textwidth]{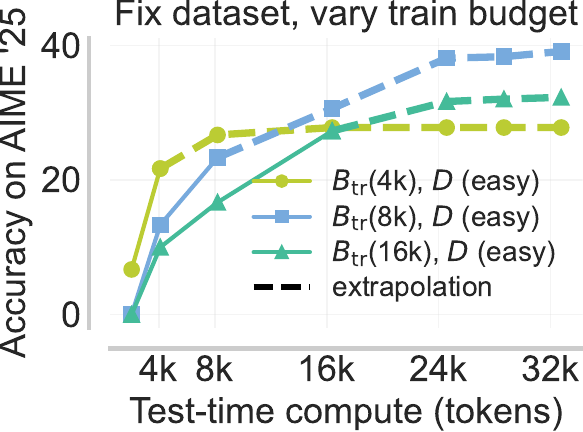}
        \caption*{\centering(a)}
    \end{subfigure}\hfill
    \begin{subfigure}[b]{0.24\textwidth}
     \captionsetup{font=small,skip=2pt}
        \centering
        \includegraphics[width=0.99\textwidth]{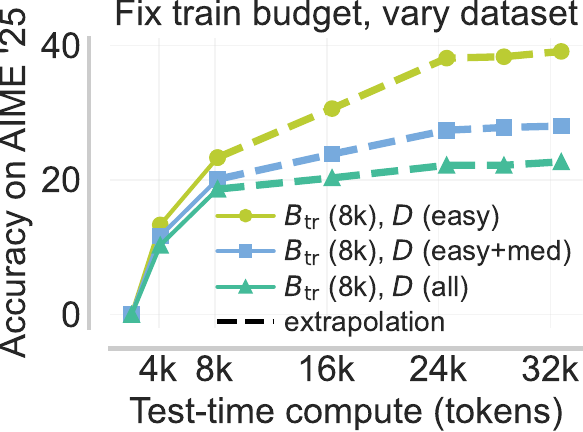}
        \caption*{\centering(b)}
    \end{subfigure}\hfill
    \begin{subfigure}[b]{0.24\textwidth}
    \captionsetup{font=small,skip=2pt}
        \centering
        \includegraphics[width=0.99\textwidth]{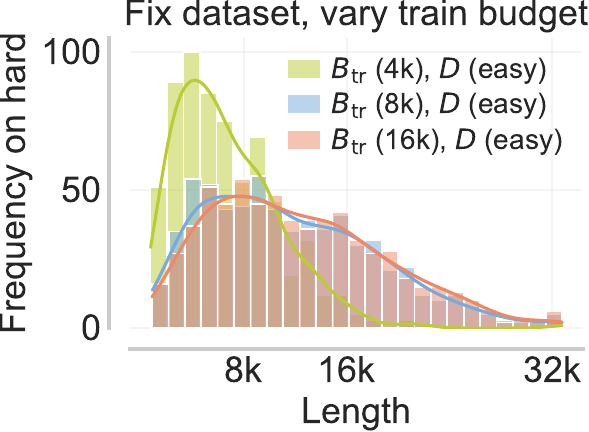}
        
        \caption*{\centering(c)}
    \end{subfigure}
        \begin{subfigure}[b]{0.24\textwidth}
    \captionsetup{font=small,skip=2pt}
        \centering
        \includegraphics[width=0.99\textwidth]{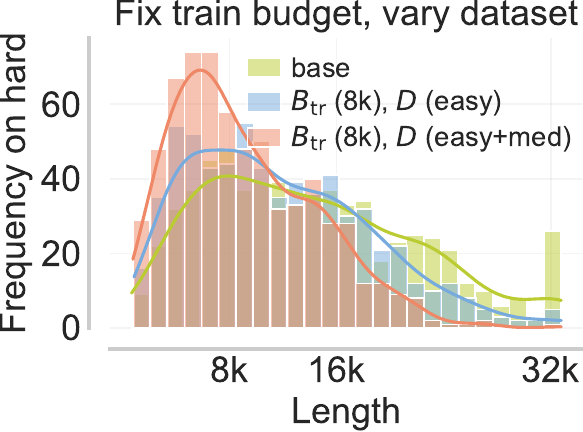}
        
        \caption*{\centering(d)}
    \end{subfigure}
    \caption{\footnotesize{{\textbf{\emph{RL training on different data and length budgets.}} \textbf{(a), (c):} Optimal results come from balancing optimization difficulty (better at shorter budgets) and in-context exploration (better at longer budgets). \textbf{(b), (d):} Training on hard problems at the 8k token budget kills longer traces with in-context exploration needed to discover solutions for hard problems}\vspace{-0.5cm}}}
    \label{fig:curriculum-panel-0}
\end{figure}

\textbf{\emph{Finding 2: Training na\"ively on a static data mixture is insufficient.}} Having identified a reasonable training budget of 8k, we now turn to studying the effect of data compositions (prompt mixtures). 
To do so, we compare the na\"ive training data mixture with equal proportions of all difficulties (easy + medium + hard) against easy, easy + medium  at $B_\mathrm{tr}$= 8k. As expected, matching train and test composition is ideal for better \emph{in-distribution} performance, \textit{i.e.}, when evaluating models at a test budget of $\btr$, equal to the training budget (see App.~\ref{sec:structured-exploration-appendix}). 
However, perhaps surprisingly, the same is not true for performance on out-of-distribution (OOD) problems, especially when performance is computed at budgets $\gg\btr$. As shown in Fig.~\ref{fig:curriculum-panel-0}(b), the model trained on \emph{only easy} problems obtains the best performance on OOD AIME'25 when extrapolating compute to 32k. This is despite the fact that AIME'25 problems resemble harder problems and few prior AIME problems are also present in the hard subset of  \mathdata{}. 

\textbf{\emph{Why does this happen?}} Given a dataset, training on budgets smaller than the length of a typical response for the base model on that dataset penalizes in-context exploration early in training. This results in overly short solutions (see Fig.~\ref{fig:curriculum-panel-0}(d)) that are mostly exploitative. When projected to our $p^k$ model from Section~\ref{sec:pk_model}, this means that at overly short budgets, RL mainly attempts to improve the failure probability $p$ of the best guess response, and does not learn to increase $k$ which corresponds to chaining asymmetries. To increase $k$, it needs to be able to learn to increase the number of attempts and requires a large enough budget. But the budget cannot be too large to result in optimization challenges.

\emph{\textbf{How can we avoid challenges with training on a fixed dataset and length budget?}}
One approach to avoid the above challenges is to incorporate a curriculum that varies $B_\mathrm{tr}$ over training. However, this alone is insufficient because, as shown above, training on hard problems with short budgets suppresses length and in-context exploration. On the other hand, we can design a curriculum over the difficulty level and keep the training budget fixed at a high enough value. However, this presents optimization challenges as we also describe above, and  leads to learning over-exploratory traces tailored to easy problems (see App.~\ref{sec:structured-exploration-appendix} for a detailed study of this on \cdown). \textbf{In a nutshell}, a curriculum that only varies the training budget or the dataset composition is insufficient to incentivize in-context exploration. To mitigate this, we describe our recipe which proposes a ``coupled'' curriculum over data composition and training budget.

\vspace{-0.2cm}
\subsection{{Our Recipe \textbf{{\texttt{e3}:  Coupled Curriculum for In-Context Exploration}}}} 
\vspace{-0.2cm}
\label{subsec:coupled-curriculum}
We develop a \termbold{coupled curriculum} that varies the training budget $\Btr$ and problem difficulty in a coordinated fashion during RL training on a base model with asymmetries. We refer to our recipe (chained asymmetries, negative gradient, and the coupled curriculum) as {\texttt{e3}}: \textit{\underline{e}xploration \underline{e}nables \underline{e}xtrapolation}.

\newcommand{\btri}{B_{\mathrm{tr},i}}
\newcommand{\btris}{B^\star_{\mathrm{tr},i}}

\begin{figure}[!t]
     \captionsetup{font=small,skip=0pt}
    \centering
    \begin{subfigure}[b]{0.32\textwidth}
    \captionsetup{font=small,skip=2pt}
        \centering
        \includegraphics[width=0.99\textwidth]{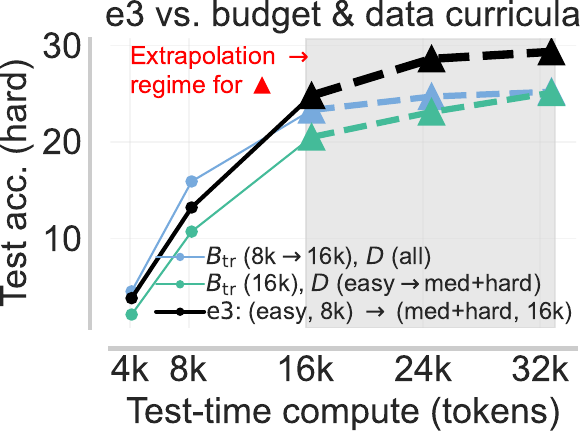}
        
        \caption*{\centering(a)}
    \end{subfigure}
    \begin{subfigure}[b]{0.32\textwidth}
    \captionsetup{font=small,skip=2pt}
        \centering
        \includegraphics[width=0.99\textwidth]{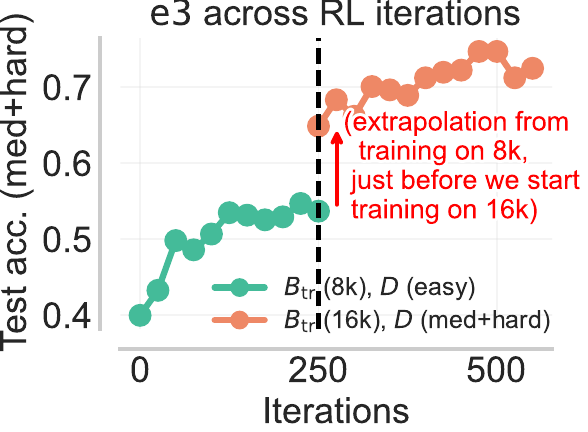}
        
        \caption*{\centering(b)}
    \end{subfigure}
    \begin{subfigure}[b]{0.32\textwidth}
    \captionsetup{font=small,skip=2pt}
        \centering
        \includegraphics[width=0.99\textwidth]{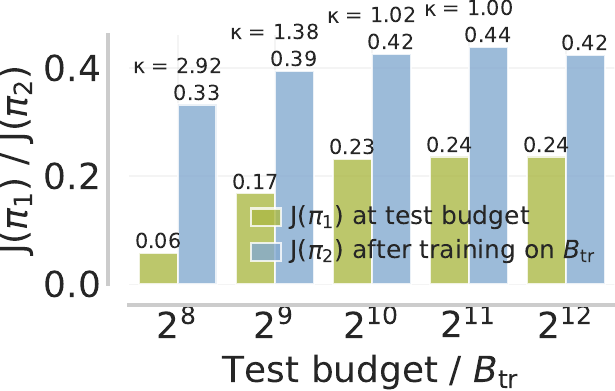}
        
        \caption*{\centering(c)}
    \end{subfigure}
\caption{\footnotesize{{\textbf{\emph{RL training with coupled curricula:}} \textbf{(a):} coupled curriculum outperforms data and budget curricula (shaded area indicates the extrapolation regime), \textbf{(b):} extrapolation gain from switching to a longer token budget of 16k on medium and hard problems, \textbf{(c):} illustrating how budget $\btr$ can be selected via Eq.~\ref{eq:opt-curr} on \cdown; $B_{\mathrm{tr}}=2^{10}$ is the smallest value with $\kappa < 1.2$, and it corresponds to the budget where accuracy plateaus. $J(\pi_1)$ is pass@128 performance and $J(\pi_2)$ is accuracy.\vspace{-0.3cm}}}}
    \label{fig:curriculum-panel-2}
\end{figure}

\textbf{\emph{Key insight for curriculum design.}} 
We simplify curriculum design by fixing the dataset at each stage and progressively increasing task difficulty in a stage-wise manner, from easy to hard. Now, at each curriculum stage~$i$, we define a dataset~$D_i$ and focus on selecting an appropriate token budget~$\btri$. The goal is to choose~$\btri$ such that training with this budget encourages \emph{in-context exploration}. That is, RL should reward longer reasoning traces that successfully chain asymmetries and are discoverable with high probability under the current model~$\pi_i$, within budget $\btri$. This ensures that the resulting policy can extrapolate to longer sequences and provides a strong initialization for the next stage~$i{+}1$, where the token budget increases to~$B_{\mathrm{tr},i+1}$.
At the same time, for optimization to be efficient, the budget~$\btri$ should be as small as possible while still accommodating most valid completions from~$\pi_i$. Balancing these desiderata, we formalize the choice of~$\btri$ via the following optimization as a thumb rule:
{
\setlength{\abovedisplayskip}{8pt}
\setlength{\belowdisplayskip}{5pt}
\begin{align}
    {\btris}(D_i) = \argmin_{B \geq B_0} B ~~~ \text{s.t.}~J(\pi_i; D_i, 2\cdot B) \; \leq \; \kappa \cdot J(\pi_i; D_i, B) , \;\; \kappa > 1
    \label{eq:opt-curr}
\end{align}}where $J(\pi ; D , B)$ denotes the performance of  $\pi$ at budget $B$ on dataset $D$,  and the budget $B_0$ denotes a reasonable minimal length for $\pi$ on dataset ${D}_i$, \textit{e.g.}, $B_0$ can be the average token length of responses from $\pi$ on $D_i$. In practice, we solve the optimization over $B$ by restricting to a fixed set of training budgets:
4k, 8k, 16k. We find the above strategy of choosing the token budget to be a useful heuristic for greedily choosing the budget $\btri$ at stage $i$ of the curriculum in a way that incentivizes in-context exploration, since it is challenging to jointly optimize the budgets across all stages. 
\textit{E.g.}, setting $\kappa=1.2$, we find 8k to be the best choice for training on easy problems (observe that the trained model satisfies the condition in \eqref{eq:opt-curr} at $\kappa=1.2$ in Fig.~\ref{fig:curriculum-panel-0}(a)). 
Following this, \ethree{} fine-tunes the Qwen3-1.7B base model on easy problems in \mathdata{} at $\btr$ of 8k, and subsequently continues training on medium and hard problems in \mathdata{} with a token budget of 16k. 
For training on medium and hard problems in \mathdata{}, we can also optimize the training budget, as we did for the run on easy problems. From Fig.~\ref{fig:curriculum-panel-0}(a), we note that the model trained with a token budget of 8k extrapolates compute to a budget of 16k and even 24k on AIME '25, after which the gains start diminishing. We find similar extrapolation performance on medium and hard problems in \mathdata{}. Thus, we can safely train on a budget of 16k or 24k on this set, and due to GPU memory constraints, we chose to train on the shorter of the two (16k).    
Finally, in Fig.~\ref{fig:curriculum-panel-1}(c), we show that the model produced by \ethree{} by training on easy problems at the end of the first stage does extrapolate well, which is helpful to kickstart RL training when we move from the budget of 8k to 16k. Concretely, we observe a >10\% performance gain with extrapolation.

\begin{table}[t]
  \centering  
  \small
  \begin{tabular*}{0.94\textwidth}{@{}%
    >{\raggedright\arraybackslash}p{0.3\textwidth}%
    @{\hskip 4pt}c@{\hskip 4pt}c@{\hskip 4pt}c@{\hskip 4pt}c@{\hskip 4pt}c@{\hskip 4pt}c%
    |@{\hskip 4pt}c@{\hskip 4pt}c@{\hskip 4pt}c@{\hskip 4pt}c@{\hskip 4pt}c@{\hskip 4pt}c@{}}
    \multirow{2}{*}{\textbf{Model}} &
      \multicolumn{6}{c}{\textbf{AIME 2025}} &
      \multicolumn{6}{c}{\textbf{HMMT 2025}} \\
      & $k{=}1$ & 2 & 4 & 8 & 16 & 32
      & $k{=}1$ & 2 & 4 & 8 & 16 & 32 \\
    \midrule
    Qwen3-1.7B~\citep{yang2025qwen3}              & 35.5 & 41.4 & 47.0 & 52.4 & 58.3 & 65.2
                                 & 22.2 & 27.3 & 33.0    & 39.5    & 46.7    & 54.9    \\
    R1-distill-Qwen-1.5B~\citep{deepseekai2025deepseekr1incentivizingreasoningcapability}     & 23.1    & 29.2    & 34.5    & 40.1    & 46.3    & 52.5
                                 & 12.5    & 19.1    & 24.3    & 27.9    & 36.1    & 42.8   \\
    Nemotron-Reasoning-1.5B~\citep{liu2025prorl} & 33.6    & 38.5    & 43.6    & 48.9    & 53.8    & 58.0
                                 & 17.4    & 22.5    & 29.6    & 35.2    & 40.7    & 45.0    \\
    \midrule
    \textbf{\texttt{e3}-1.7B (Ours)}                           & \textbf{43.8} & \textbf{51.1} &\textbf{ 56.7} & \textbf{60.8} &\textbf{ 64.0 }&\textbf{ 67.2}
                                 & \textbf{24.7} &\textbf{ 30.4} & \textbf{37.0} & \textbf{44.1} & \textbf{50.8} & \textbf{56.1} 
  \end{tabular*}
  \vspace{-0.2cm}
  \caption{\footnotesize \textbf{\textit{Final results with \ethree{}: Best $<$2B sized model on AIME/HMMT'25}}: We measure \texttt{pass@k} (\%) on AIME'25 and HMMT'25 for our 1.7B model obtained by post-training the Qwen3-1.7B base model on \mathdata{} with our recipe \ethree{}. Following Sec.~\ref{subsec:coupled-curriculum}, we use a coupled task and budget curriculum during RL training (first train on easy problems at $\btr$=8k, and then on medium and hard ones at $\btr$=16k). We compare the gains with the base model and other strong reasoning models withing the $<$2B model family. Note that unlike recent trends~\cite{yue2025does} that show RL training improving \texttt{pass@1} at the cost of \texttt{pass@k} for a higher $k$, we note that \ethree{} trained models improve performance by not just sharpening the base model distribution around high reward traces, but by actually chaining asymmetries and discovering new solutions with longer traces.}
  \label{tab:passk}
  \vspace{-0.3cm}
\end{table}

\begin{wrapfigure}{r}{0.35\textwidth}
\vspace{-0.9cm}
  \begin{center}
    \hspace{-0.3cm}
    \includegraphics[width=0.97\linewidth]{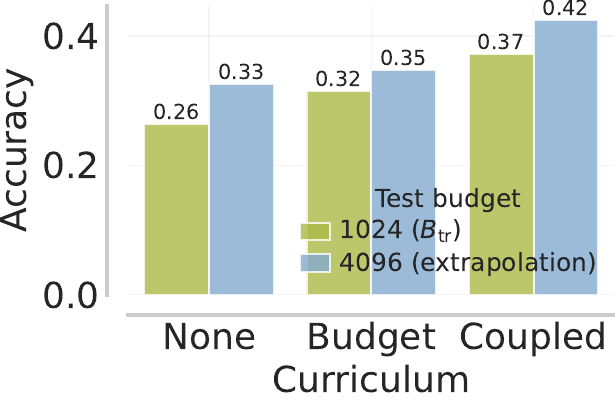}
  \end{center}
  \vspace{-0.6cm}
\label{fig:countdown-curriculum-mainpaper}
\vspace{-0.3cm}
\end{wrapfigure}
\underline{\termbold{\textbf{Illustrating the efficacy of coupled curriculum.}}} We first demonstrate the efficacy of \ethree{} on \cdown{} by training on problems of 3, 4, and 7 numbers. In our coupled curriculum, we first train on problems of easier difficulty with 3 and 4 numbers, on a budget of 256 tokens. Following~\eqref{eq:opt-curr} with say $\kappa = 1.2$, to select the budget for the next stage, we examine the performance of this first-stage model on the second-stage dataset consisting of harder problems with 7 numbers. \eqref{eq:opt-curr} prescribes that we pick the smallest but reasonable $B$ such that there is only a marginal improvement from extending $B$ to $2B$, defined by $\kappa = 1.2$. As shown in Fig~\ref{fig:curriculum-panel-2}(c), this corresponds to ${B^\star_{\mathrm{tr},1}}(D_{1})= 1024$ (where $D_{1}$ is the second stage training dataset). Indeed, Fig~\ref{fig:curriculum-panel-2}(c) shows that at ${B^\star_{\mathrm{tr},1}}(D_{1})= 1024$, we get nearly the best extrapolation performance to 4096 tokens on the harder problems (7 numbers). We also note that while $\btr=2048$ marginally improves test performance over $\btr=1024$, it is unclear apriori if $\btr=2048$  would train stably and our goal is to make a thumb rule prescription. We also find in the figure on the right, that our coupled curriculum outperforms budget curriculum or not training with any curriculum.

\vspace{-0.2cm}
\subsection{Final Results with \texorpdfstring{\ethree{}}{e3}: \textit{State-of-the-art} \texorpdfstring{$<$2B}{<2B} Model on AIME/HMMT'25}
\vspace{-0.2cm}

\textbf{\textit{Extrapolation to 32k with \ethree{}}.} In Fig.~\ref{fig:curriculum-panel-1}(a,b), we compare the performance of a Qwen3-1.7B model fine-tuned using \ethree{} with open-source models, including some 7B and 32B models. As shown, at a test-time token budget of 32k tokens, \ethree{} achieves state-of-the-art performance on AIME'25 and HMMT'25, within a model class of size $<$2B. We outperform the best model in this class by $>$8\% on AIME'25 in terms of peak performance, and show that our model, trained only up to a budget of 16k, extrapolates better than other models including s1.1-32B~\cite{muennighoff2025s1simpletesttimescaling} and OpenThinker-7B~\cite{openthoughts} when we extrapolate them to 32k output tokens.
In principle, one can simply force the model (trained even with SFT) to use more test-time compute by intervening its output trace with an appended prompt (\textit{e.g.}, by appending ``Wait'' to an output trace as suggested in s1~\cite{muennighoff2025s1simpletesttimescaling}).  
Interestingly, Fig.~\ref{fig:curriculum-panel-2}(c) shows that compared to budget forcing via ``Wait'', \texttt{e3} achieves substantially better scaling, without any form of prompting or budget forcing.

\textbf{\textit{Improving \texttt{pass@32} with \ethree{}}.} In Tab.~\ref{tab:passk}, we also report the \texttt{pass@k} performance, comparing \ethree{} with other models of a similar size. We find that our final model at the end of second stage of training on a budget of 16k outperforms other models on higher values of $k$, on AIME and HMMT '25. We  especially note the comparison against the Nemotron-Reasoning-1.5B model~\citep{liu2025prorl} trained with a prolonged RL training recipe on a broader dataset, including our training data.
This model consistently improves \texttt{pass@16} performance during RL training~\cite{liu2025prorl}. To concretely describe our estimation procedure, we used 128 rollouts per prompt to compute a bootstrapped estimate~\cite{chen2021evaluating} of the \texttt{pass@k} performance for $k=1,2,\ldots,32$. Evaluating \texttt{pass@k} at higher values of $k$ would require a much higher number of rollouts ($>$$2048$) since the variance of the \texttt{pass@k} estimate increases sharply with $k$, for a given number of rollouts. Moreover, in all of our GRPO training runs, we only use a maximum of 32 rollouts per problem to estimate the advantage value. Therefore, we can conclude that if \ethree{} is able to improve over \texttt{pass@32} of the base model, then  it does improve beyond any na\"ive distillation-based approach, that aims to distill the \texttt{pass@32} policy corresponding to the base model into a better \texttt{pass@1} policy.

\begin{figure}[!t]
     \captionsetup{font=small,skip=0pt}
    \centering
    \begin{subfigure}[b]{0.33\textwidth}
    \captionsetup{font=small,skip=2pt}
        \centering
        \includegraphics[width=0.99\textwidth]{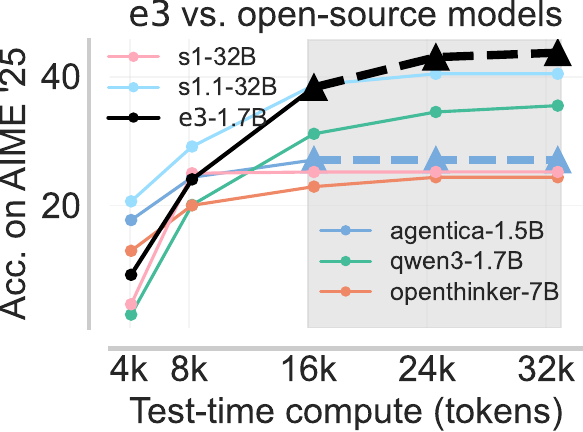}
        \caption*{\centering(a)}
    \end{subfigure}\hfill
    \begin{subfigure}[b]{0.33\textwidth}
     \captionsetup{font=small,skip=2pt}
        \centering
        \includegraphics[width=0.99\textwidth]{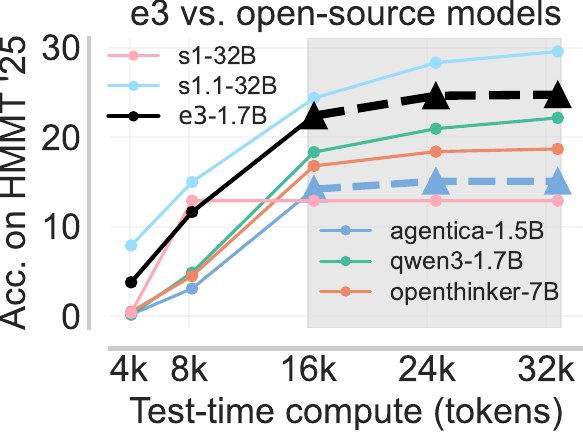}
        \caption*{\centering(b)}
    \end{subfigure}\hfill
    \begin{subfigure}[b]{0.33\textwidth}
    \captionsetup{font=small,skip=2pt}
        \centering
        \includegraphics[width=0.99\textwidth]{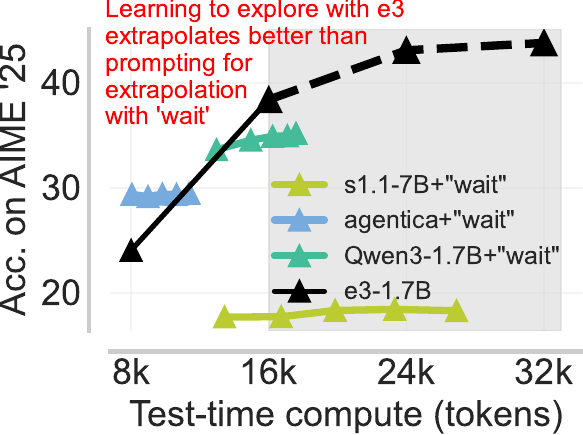}
        
        \caption*{\centering(c)}
    \end{subfigure}
        
\caption{\footnotesize{{\textbf{\emph{RL training with coupled curricula (\ethree{}).}} The shaded area indicates the extrapolation regime and dashed curves indicate that we are testing the model beyond the training budget (shown with the solid line). \textbf{ (a), (b):} \ethree{} achieves \emph{state-of-the-art} performance across models < 2B, and outperforms some larger sized models like openthinker-7B~\cite{openthoughts} and s1.1-32B~\cite{muennighoff2025s1simpletesttimescaling} (on AIME'25) on larger test-time budgets upto 32k, \textbf{(c):}  \texttt{e3} (w/o ``wait'') extrapolates better to larger test-time token budgets, compared to budget forcing with ``wait'' prompt 2-8 times, as proposed in s1~\cite{muennighoff2025s1simpletesttimescaling}.}\vspace{-0.2cm}}}
    \label{fig:curriculum-panel-1}
\end{figure}

\vspace{0.2cm}

\begin{AIbox}{Takeaways: Coupled data \& budget curriculum structures exploration during training.}
\begin{itemize}[leftmargin=0.7em]
    \setlength\itemsep{0em}
    \item RL with fixed $\btr$, $D$ hurts in-context exploration: (i) short $\btr$ penalizes exploration on hard problems as budget is overrun; (ii) large $\btr$ overfits on over-exploratory behavior on easy ones.
    \item We propose a coupled curriculum \ethree{}: at each stage, given $D$,  choose smallest $\btr$ such that chaining more asymmetries till a budget of $2\cdot\btr$ is positively rewarded at RL initialization.
    \item By fine-tuning Qwen3-1.7B with \ethree{}, we outperform $<$2B models on AIME'25, HMMT '25.
\end{itemize}
\end{AIbox}

\vspace{-0.3cm}
\section{Related Work}
\label{sec:relwork}
\vspace{-0.1cm}
\textbf{Scaling test-time compute via long CoT reasoning.} Prior work explores a number of avenues for scaling test-time compute, including majority voting~\citep{wang2022self}, best-of-n sampling, and beam search~\citep{setlur2024rewarding, snell2024scaling}, as well as sequential self-correction~\citep{qu2024recursive, kumar2024training}. More recent results indicate that training models to use test-time compute to generate longer chains of thought (CoT) that combine verification, search, and self-correction -- all in a free-form manner, performs better~\citep{deepseekai2025deepseekr1incentivizingreasoningcapability, kimiteam2025kimik15scalingreinforcement, openai2024openaio1card}, resulting in widespread open-source reproduction efforts~\citep{openr1, yeo2025demystifyinglongchainofthoughtreasoning, zeng2025simplerlzooinvestigatingtamingzero,deepscaler2025}. We situate our work in the paradigm of long CoT reasoning.

\textbf{Test-time extrapolation. }The true benefit of test-time scaling is consistently improving performance as we extrapolate test compute. While prior work tests the model's performance on budgets longer than the training budget~\cite{zeng2025simplerl, deepcoder2025}, they do not explain the relationship between the training recipe and the extrapolation, like we aim to do in our work. Other works perform extrapolation by explicitly prompting models to generate more tokens when a response terminates~\cite{muennighoff2025s1simpletesttimescaling,aggarwal2025l1}, whereas, we show that models that learn to explore in-context extrapolate test compute better than prompting-based approaches (Fig.~\ref{fig:curriculum-panel-1}). In particular, we study the role played by the base model, training algorithm (RL), as well as data mixtures and token budgets, on the ability to extrapolate. Furthermore, prior work~\cite{setlur2025scaling} has investigated scaling when train and test budgets are the same, but we expand the scope of this comparison substantially.

\textbf{Exploration in test-time scaling.} 
While prior works have shown the importance of the base model’s ability to conduct exploration~\citep{gandhi2025cognitivebehaviorsenableselfimproving, liu2025understanding}, we discover that it is crucial for extrapolation. We show that the negative gradient in RL incentivizes chaining multiple asymmetries and leads to longer response length, and better performance. SFT alone does not provide this kind of chaining or exploration benefits. Our analysis is orthogonal to theoretical works \citet{setlur2025scaling,swamy2024minimaximalist}, which shows that RL performs better than SFT, but from a statistical perspective, whereas our argument is more focused on the learning dynamics. Concurrent work builds techniques to boost 
exploration during RL via advantage normalization \citep{li2022understanding,yu2025dapo} or PPO clipping \citep{yu2025dapo}, and these techniques can be combined with \ethree, but they do not study the role of negative gradients in learning to explore. Finally, \citet{wang2025reinforcement} briefly remarks about the role of policy gradient loss and entropy when running RL with only a few examples. Our study investigates the underlying mechanism of negative gradients increasing length and entropy.

\textbf{Data and length curricula. }
Recent works have also investigated using a curriculum on problem difficulty \cite{kimiteam2025kimik15scalingreinforcement, xie2025logic, shi2025efficientreinforcementfinetuningadaptive} and output length \cite{deepscaler2025, acemath2024} during RL training. Their motivation stems primarily from an efficiency standpoint: avoiding zero advantage updates \cite{shi2025efficientreinforcementfinetuningadaptive, yu2025dapo}, efficient optimization \cite{deepscaler2025}, or efficiency of using test-time compute \cite{qu2025optimizing}. While we do make similar observations regarding each curriculum individually, perhaps our most interesting finding is that carefully coupling both data and budget curricula can lead to much better performance and extrapolation, beyond merely some gains in efficient training. We show that training on hard problems with short budgets often yields terse solutions that fail to extrapolate, while easy problems with long budgets can cause optimization issues or verbose outputs. Thus, curricula must be carefully designed to support effective extrapolation. Conceptually, our curricula are most related to dense progress rewards~\citep{qu2025optimizing,setlur2024rewarding}, in the sense that curricula incentivize different degrees of progress for different questions, at different points in training. 

\vspace{-0.2cm}
\section{Discussion and Conclusion}
\vspace{-0.1cm}
\label{sec:conclusion}
We show that in-context exploration is a core capability to enable extrapolation of test-time compute in LLMs. Therefore we build a recipe that amplifies in-context exploration. Our recipe \ethree{}, leverages (1) asymmetries in the base model, (2) negative gradients during RL training, and (3) a coupled curriculum over data and token budget to train a model that can  perform in-context exploration. Applied to the Qwen3-1.7B model, our method achieves state-of-the-art performance on the AIME/HMMT'25 benchmarks, with particularly strong gains in the extrapolation regime. We also show that our \ethree{} recipe also improves \texttt{pass@k} over the course of training, for values of $k$ upto 32 that we evaluate. 
There are a number of implications of our work and a number interesting directions that future work should build upon. We list the main technical implications and open questions below.
\begin{itemize}[itemsep=4pt]
\vspace{-0.3cm}
    \item \textbf{Sharpening vs in-context exploration.} A number of concurrent RL results either directly~\citep{yue2025does} or indirectly~\citep{sharpen1, sharpen2, sharpen3, sharpen4} argue that RL training on LLMs sharpens the base model's distribution, as also previously studied by \citet{huang2024self}. In contrast to this, our study shows that if we can utilize a coupled curriculum on top of a base model that admits asymmetries, RL can actually enable chaining new asymmetries, resulting in an increase in length, indicating the presence of structured exploration. This behavior is distinct from traditional sharpening that corresponds to cloning one (or few) of the responses sampled from the base model. In fact, our conceptual study in the $p^k$ model in Section~\ref{sec:pk_model} also highlights these two distinct phases during RL: an initial in-context exploration phase where negative gradients lead to an increase in response length and the policy learns to utilize test-time compute for better exploration, followed by a phase where it sharpens to the best traces found thus far. The design of \ethree{} enables it to operate in the former phase. We believe concurrent works that finds RL largely sharpens the model operate in the second regime by training on data that does not require chaining asymmetries or operating with a very low training budget such that chaining is impossible. As a result, models trained purely in the sharpening regime may behave similarly to the base model with an alternate prompt, with RL perhaps offering little more than an implicit prompt tuning effect. But we would not expect this for the chaining regime. A detailed study on separating these regimes, and identifying all the factors that draw RL training into these regimes is an interesting direction for both theoretical and empirical research.
    \item \textbf{Connection with dense progress rewards.} While \ethree{} utilizes a coupled curriculum, this curriculum is closely connected with the use of dense rewards, as prescribed by our prior work~\citep{qu2025optimizing,setlur2024rewarding}. To see why, note that one can reparameterize coupled curriculum into a single round of training with dense rewards applied to short segments of the output response, perhaps in a similar way as \citet{qu2025optimizing,qi2025optimizing}. Therefore, the success of the coupled curriculum approach in \ethree{} at improving performance and not only in reducing total training compute perhaps hints at future success with dense rewards at scale, with initial results showing that dense rewards help larger models already being shown in the community~\citep{wang2025beyond}. We encourage readers to explore the connection between curriculum and dense rewards further.  
    \item \textbf{Introducing new asymmetries.} The conceptual model behind \ethree{} applies with any asymmetry, though most experiments in this paper utilize only the verification-generation gap. It would be interesting to identify other asymmetries and study methods to imbue base models with these asymmetries. Definition~\ref{def:asymmetry} in Section~\ref{sec:vg-gap} provides a starting point to define these asymmetries.
    \item \textbf{Is curriculum fundamentally needed?} A natural question is whether curriculum is fundamentally necessary as we vary model sizes and capabilities. Unlike supervised learning on a fixed dataset, online RL generates its own rollouts. Reinforcing chaining behavior via negative gradients (Sec.~\ref{sec:negative-gradient}) requires that such chaining reliably improves performance on training problems much more substantially compared to sampling diverse traces that do not chain asymmetries. This likely necessitates specific training configurations regardless of model size with standard outcome-reward RL, or the use of dense rewards (as discussed above). While larger models may admit simpler curricula, deliberately using currciulum or dense rewards as inspiration may be critical.
    \item \textbf{Explicit exploration bonuses.} In our runs, the main issue hindering us from benefits of further scaling of output length during RL is the repetition bias in the base model, where it tends to repeat previously-generated segments in its trace beyond a certain output length. This repetition bias inhibits the efficacy of in-context exploration beyond a certain output length and as a result inhibits further test-time scaling. We believe that explicit exploration bonuses that enable the model to search for tokens in this regime would result in even better in-context exploration.
    \vspace{-0.3cm}
\end{itemize}
Finally, our study is limited in terms of model scale and domain. Future work should explore how \ethree{} generalizes to larger model scales and other reasoning domains.

\vspace{-0.2cm}
\section*{Acknowledgements}
\vspace{-0.2cm}
We thank Christina Baek, Yuxiao Qu, Anikait Singh, Yoonho Lee, Max Sobol Mark, Zheyuan Hu, Seohong Park, Bhavya Agrawalla, Sang Michael Xie, Paria Rashidinejad, and the rest of the AIRe lab at CMU for informative discussions, feedback, input on our results, and a previous version of this paper. We thank Yuxiao Qu for help with debugging implementations and infrastructure. The main large-scale experiments in this paper utilized H100 GPU resources from the Orchard cluster in the FLAME center at CMU for which we especially thank Graham Neubig and Chenyan Xiong for their generous support, and TPUs from Google Cloud. We thank Oumi for providing us with resources that supported the experiments on the Countdown domain. This project is supported by funding from the Office of Naval Research under N00014-24-1-2206 and a Schmidt Sciences AI2050 Fellowship. AS is supported by JP Morgan PhD fellowship. This paper does not reflect the opinions of employers or other parties.

\bibliography{references}

\newpage

\appendix
\part*{Appendices}
\vspace{-0.2cm}

 \hyperref[sec:test-extrapolation-appendix]{\textbf{A.}}$~~~~$ Testing Extrapolation of Open Source Models. \\
\hyperref[sec:vg-gap-appendix]{\textbf{B.}}$~~~~$ Additional Experiments and Details for Section~\ref{sec:vg-gap} (Chained Asymmetries). \\
\hyperref[sec:negative-gradient-appendix]{\textbf{C.}}$~~~~$ Additional Experiments and Details for Section~\ref{sec:negative-gradient} (Negative Gradient). \\
\hyperref[sec:structured-exploration-appendix]
{\textbf{D.}}$~~~\,$ Additional Experiments and Details for Section~\ref{sec:structured-exploration} (Curricula Training). \\
\hyperref[sec:omitted-proofs-appendix]{\textbf{E.}}$~~~\:$ Omitted Proofs. \\
\hyperref[sec:broad-impact-app]{\textbf{F.}}$\;\;\,\,\,\,$ Broader Impact Statement.  \\
\hyperref[sec:computational-resources-appendix]{\textbf{G.}}$~~~\:$ Note on Computational Resources Used for \texttt{e3}. \\
\hyperref[sec:example-traces]{\textbf{H.}}$~~\:\,$ Example Traces.

\vspace{-0.3cm}

\section{Testing Extrapolation of Open-Source Models}
\label{sec:test-extrapolation-appendix}

 \begin{figure}[htbp]
  \centering
  \includegraphics[width=0.99\linewidth]{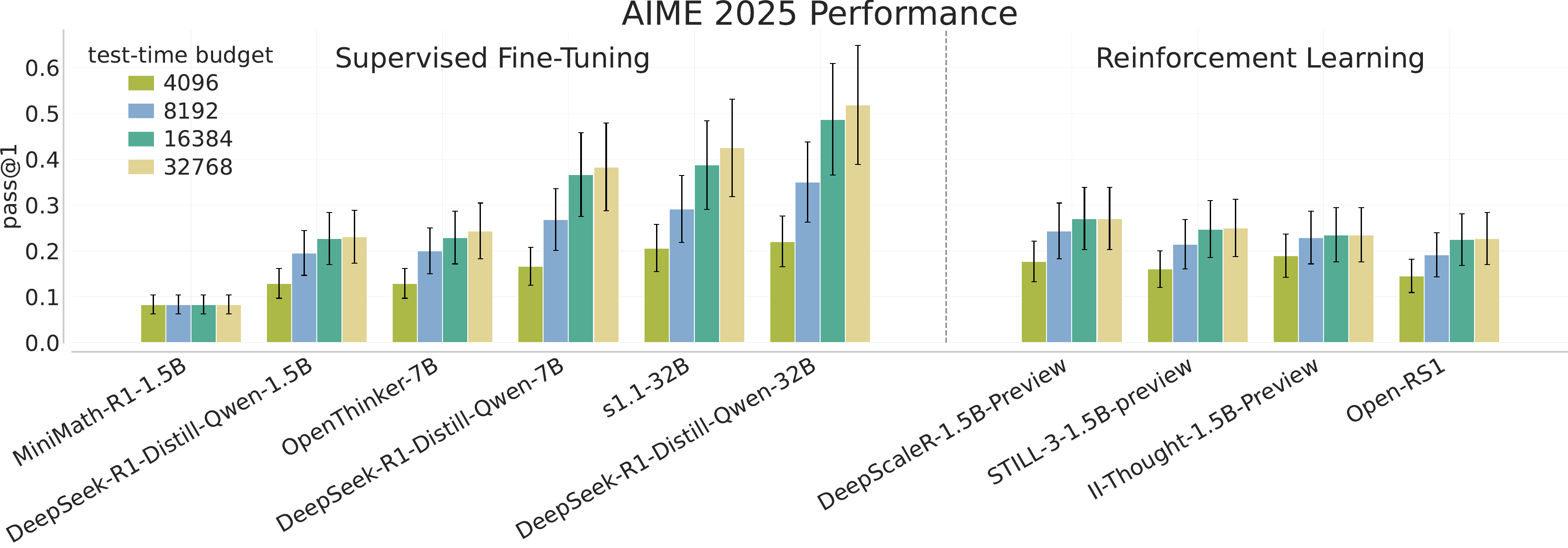}
  \vspace{0cm}
  \caption{\footnotesize\textbf{\emph{Extrapolation of test-time compute:}}{We plot the performance (pass@1) on AIME 2025 at different test-time compute budgets across multiple open-source models of different sizes, trained with SFT or RL.} }
  \label{fig:aime}
  \vspace{0cm}
 \end{figure}

\textbf{Extrapolation on AIME 2025.}
Extrapolation (i.e. the chaining of generation, verification, refinement, etc.) can potentially extend LLM performance after training, and do so beyond the context length the model was originally trained on. To evaluate this properly, we need sufficiently challenging problems that allow meaningful expressiveness in reasoning beyond small context lengths. The math problems associated with AIME align with this, and our evaluations prioritize AIME 2025 to attempt to mitigate any potential data contamination in the models' training sets from previous years of AIME. The goal of the experiment is to measure the extent to which test-time compute influences overall model performance as context length increases, with the expectation that increasing output length allows models to "reason" for longer periods, continuing the extrapolation process, and ultimately arriving at the correct answer more frequently.

\textbf{Experiment setup.}
Inference for every open-source model was performed using Oumi through data-parallel SGLang. All models had inference run with a max output length of approximately 32k tokens, though some are slightly lower due to this exceeding their max context length when combined with the prompt. The exact inference hyperparameters are described in Table \ref{aime-exp-table}. After inference, the model responses were truncated from the right side until the number of remaining tokens present was equal to the specified test-time budget. 16 responses were collected for every problem in AIME with the specified inference settings, and the Pass@1 rate was calculated by averaging over these 16 responses. Final answers were extracted using a regular expression for the boxed portion of the answer, with correct answers marked as passing and incorrect or incorrectly parsed answers marked as nonpassing. The prompt used is in Box~\ref{box:aime-eval-prompt}, and the problems were taken from the FVU AIME 2025 dataset on HuggingFace\footnote{ \url{https://huggingface.co/datasets/FVU/AIME_2025}}.

\newtcolorbox[auto counter, number within=section]{instructionbox}[2][]{%
 title=Box~\thetcbcounter: #2,
 label=#1,
 breakable
}

\begin{instructionbox}[box:aime-eval-prompt]{AIME Evaluation Prompt Template}
You will be given a math problem. Solve the problem step by step. Output your final answer in the form of \textbackslash\textbackslash boxed\{your answer\}.
Problem: \{problem\}
 
\end{instructionbox}

\begin{table}[htbp]
\centering
\begin{tabular}{p{0.35\linewidth}|p{0.06\linewidth}|p{0.06\linewidth}|p{0.09\linewidth}|p{0.10\linewidth}|p{0.12\linewidth}}
Model             & Temp. & Top $p$ & Rollouts & Max New Tokens & Model Max Length \\ \midrule
MiniMath R1-1.5B       & 0.6     & 0.95 & 16          & 32768     & 40960      \\
DeepSeek R1-Distill-Qwen-1.5B & 0.6     & 0.95 & 16          & 32768     & 40960      \\
OpenThinker-7B        & 0.6     & 0.95 & 16          & 31000     & 32768      \\
DeepSeek-R1-Distill-Qwen-7B  & 0.6     & 0.95 & 16          & 32768     & 40960      \\
s1.1-32B           & 0.6     & 0.95 & 16          & 31000     & 32768      \\
DeepSeek-R1-Distill-Qwen-32B & 0.6     & 0.95 & 16          & 32768     & 40960      \\
DeepScaleR-1.5B-Preview    & 0.6     & 0.95 & 16          & 32768     & 40960      \\
STILL-3-1.5B-preview     & 0.6     & 0.95 & 16          & 32768     & 40960      \\
II-Thought-1.5B-Preview    & 0.6     & 0.95 & 16          & 32768     & 40960      \\
Open-RS1           & 0.6     & 0.95 & 16          & 32768     & 40960      \\
\end{tabular}
\newline
\vspace{0.1cm}
\caption{Inference parameters used for generating the extrapolation plots in Figure~\ref{fig:aime-line}.}
\label{aime-exp-table}
\end{table}

\textbf{Results.}
The results in Figure \ref{fig:aime} show that as the maximum number of output tokens increases, every model capable of "reasoning" is able to attain a higher Pass@1 rate, with performance generally saturating at 16k tokens with relatively minor improvements at 32k. We do not observe this with MiniMath-R1-1.5B, and we suspect this is due to its fine-tuning focusing solely on smaller math problems trained with supervised fine-tuning, likely resulting in catastrophic forgetting of the ability to continuously extrapolate. Interestingly, we do not see a strong improvement in extrapolation behavior among models tuned with reinforcement learning compared to DeepSeek R1-Distill-Qwen-1.5B, which was trained with supervised fine-tuning. We suspect that this is likely due to the nature of the distillation data from the R1 model, which, if varied sufficiently in length, could avoid the length bias normally learned from supervised fine-tuning, while still teaching the model to perform extrapolation.

\section{Additional Experiments and Details for Section~\ref{sec:vg-gap} (Chained Asymmetries)}
\label{sec:vg-gap-appendix}

\subsection{Details on \texorpdfstring{\mult{}}{Mult} and \texorpdfstring{\multv{}}{Mult-V}}

\textbf{Data collection.} Both \mult{} and \multv{} consist of multiplication traces for solving a 5-digit $\times$ 5-digit multiplication problem. For the \mult{} task, we use a Llama3.2-3B instruction tuned model where the number of intermediate verification attempts is much lower in a trace when asked to solve a multiplication problem. In fact, it is not hard to see that, in general, for multiplication, generation of a trace may be as hard as verifying a generated one, as the only way to verify the entire trace is to re-attempt the multiplication or carry out a division with the computed target. We contrast this task with the \multv{} task, where the Llama3.2-3B models are first finetuned on traces from Qwen-32B-R1-Distilled and GPT-4o models. These traces contain multiple verification attempts that verify intermediate steps solving smaller multiplication problems, and the steps are part of an entire trace that attempts to solve the main multiplication problem involving two 5-digit numbers. For collecting data we used the prompt in Box~\ref{box:multv-prompt}. In App.~\ref{sec:example-traces} Example 2, we also provide an example multiplication trace with verification attempts sampled by the base model in \multv{}. As we will see in Fig.~\ref{fig:neg-grad-multiply}, the absence of asymmetries in \mult{} leads to lower accuracy and verifications when compared to \multv{}, where asymmetries are present.  

\begin{instructionbox}[box:multv-prompt]{Prompt for generating \multv{} data}
Multiply \{num1\} and \{num2\}. Please reason step by step, and put your 
final answer within \textbackslash\textbackslash boxed\{\}. At each step, try to verify your response if possible and prefix the line with ``Check:''. <think> 
\end{instructionbox}

\begin{table*}[!ht]
\centering
\begin{tabularx}{0.4\linewidth}{l|c}
 \toprule
\multicolumn{1}{c}{\textbf{Hyperparameter}} \vline & \multicolumn{1}{c}{\textbf{Values}} \\ 
\midrule
train\_batch\_size & 256 \\
ppo\_mini\_batch\_size & 64 \\
learning\_rate & 5.0e-6 \\
kl\_loss\_coef & 0.001 \\
entropy\_coeff & 0.001 \\
temperature & 1.0 \\
rollout.n & 16 \\
ppo\_lowerclip\_threshold & 0.2  \\
ppo\_higherclip\_threshold & 0.2  \\
\end{tabularx}
\vspace{0.1cm}
\caption{Verl~\cite{verl} hyperparameters used for \mult{} and \multv{}.}
\label{tab:finetune_hyper_mult}
\end{table*}

\textbf{Training details.} Hyperparameters for our experiments on \mult{} and \multv{} are given in Table~\ref{tab:finetune_hyper_mult}.

\subsection{Details on \texorpdfstring{\cdown{}}{CDown}}
\textbf{Training details.}
Hyperparameters in \cdown{} experiments follow the table below unless otherwise specified. In all of our \cdown{} experiments, we take the fine-tuned Llama3.2-3B base model from~\cite{gandhi2025cognitivebehaviorsenableselfimproving}. For Fig.~\ref{fig:vg-gap-main}, we trained with $\btr = 512, 1024, 2048$ on problems with $3,4,5,6$ candidates. The total number of datapoints we used was 40000, which were evenly split across the four difficulties.
\begin{table*}[ht]
\centering
\begin{tabularx}{0.4\linewidth}{l|c}
 \toprule
\multicolumn{1}{c}{\textbf{Hyperparameter}} \vline & \multicolumn{1}{c}{\textbf{Values}} \\ 
\midrule
train\_batch\_size & 128 \\
ppo\_mini\_batch\_size & 32 \\
learning\_rate & 1.0e-6 \\
kl\_loss\_coef & 0.001 \\
entropy\_coeff & 0 \\
temperature & 0.6 \\
rollout.n & 8  \\
ppo\_lowerclip\_threshold & 0.2  \\
ppo\_higherclip\_threshold & 0.2  \\
\end{tabularx}
\vspace{0.1cm}
\caption{Verl~\cite{verl} hyperparameters used for \cdown{}.}
\label{tab:finetune_hyper}
\end{table*}

\textbf{Evolution of chained asymmetries at test time.}
In Fig.~\ref{fig:base-asymmetries-full}, we show that as training progresses, responses with more chained asymmetries enjoy a greater improvement. If we move across any diagonal parallel to the main diagonal from top left to bottom right, we move across a constant attempt budget (\textit{e.g.,} moving from $16 \text{ chained asymmetries} \times 1 \text{ pass}$ to $8 \text{ chained asymmetries} \times 2 \text{ passes}$). Having sequential chained asymmetries become increasingly better than parallel rollouts as training progresses, indicating the exploitation of asymmetries in RL training. See example of chained asymmetry in App.~\ref{sec:example-traces}, Example 1.

\begin{figure}[!ht]
  \begin{center}
  \hspace{-0.4cm}
  \includegraphics[width=0.96
  \linewidth]{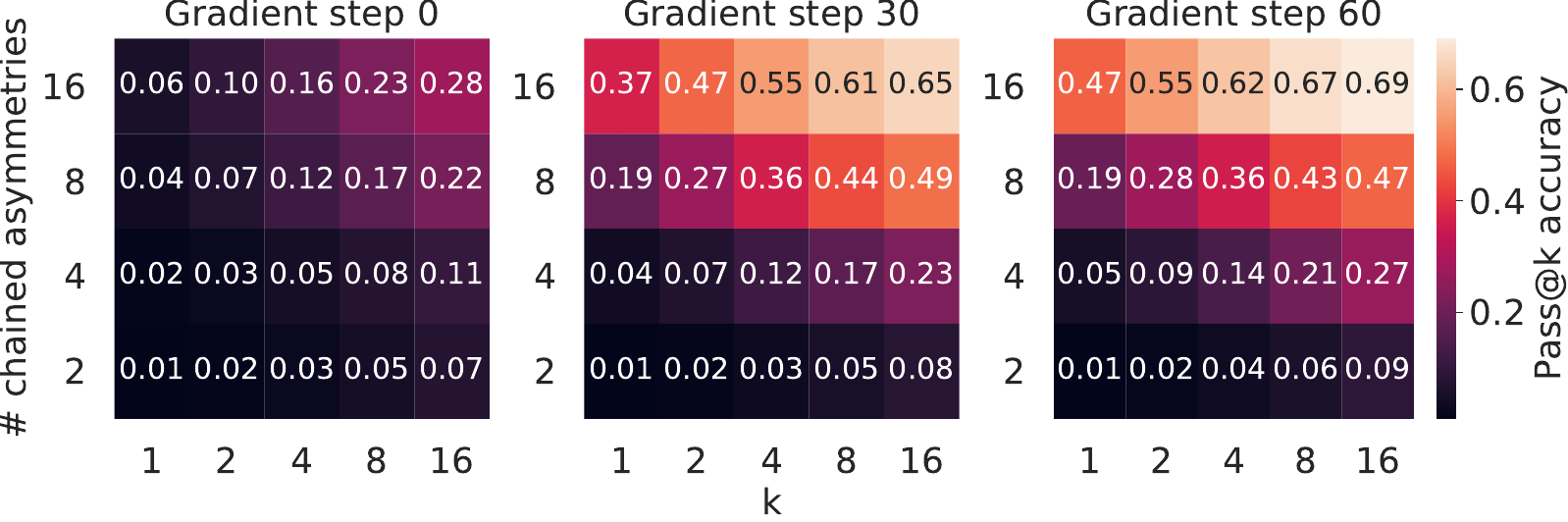}
  \vspace{-0.4cm}
 \end{center}
  \caption{\footnotesize{\textbf{\emph{Evolution of asymmetries during training on \cdown}:}} More chained asymmetries lead to a greater improvement in pass@k performance across gradient steps.
 \vspace{-0.4cm}
 }
  \label{fig:base-asymmetries-full}
\end{figure}

\subsection{In the Presence of Asymmetries, KL Divergence with Base LLM Reduces as Training Token Budget Increases}

\begin{figure}[!h]
    \centering
    \includegraphics[width=0.7\linewidth]{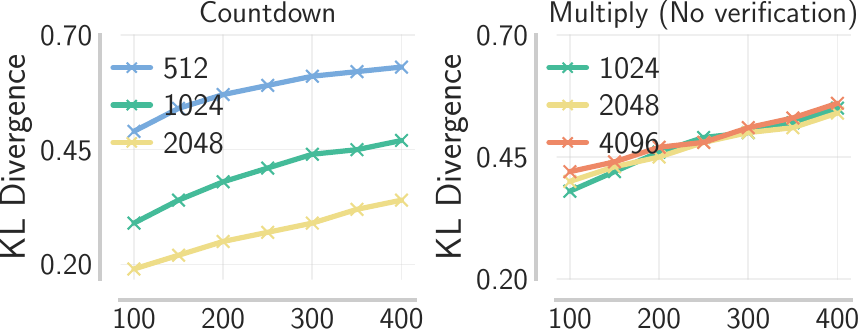}
    \caption{\footnotesize{\textbf{\emph{KL-divergence with base LLM on \cdown{} and \mult{}}:} When running RL training on \cdown{} and \mult{} with multiple training budgets (512, 1024, 2048 on \cdown{} and 1024, 2048, 4096 on \mult{}) we note that the KL divergence is lower when running RL training on higher training budgets, when the base model presents asymmetries (in here the asymmetry is given by the verification generation gap on \cdown{}).}}
    \label{fig:kl-asymmetries-appendix}
\end{figure}

In Fig.~\ref{fig:kl-asymmetries-appendix}, we  interestingly observe that training with higher $B_\mathrm{tr}$ results in a smaller token KL-divergence from $\pi_b$ all throughout training on countdown. On multiplication in the absence of asymmetries, the KL-divergence values are roughly similar for all $B_\mathrm{tr}$. This means that when the verification-generation asymmetry is present, the training process deviates less from $\pi_b$ at each token, but is able to ``chain'' multiple verification and generation attempts together to improve accuracy, by learning to explore over the space of basic skills. Prior work argues that a model that deviates less from the base pre-trained model generalizes better on unseen prompts~\citep{gao2019reward}. If we were to apply this argument in our case, this means that models that are able to use asymmetries better should result in better performance on unseen prompts, especially when operating at higher test compute.

\section{Additional Experiments and Details for Section~\ref{sec:negative-gradient} (Negative Gradient)}
\label{sec:negative-gradient-appendix}

\subsection{Details for \texorpdfstring{\cdown{}}{CDown}}
We trained models for 90 steps on problems with 5 candidate numbers with a training budget of 2k.

\textbf{Cumulative unique attempts plot.} Fig.~\ref{fig:unique-attempts} (left) was filtered on incorrect traces on problems with < 50\% success across gradient steps. We select only incorrect traces to capture the ability of the model to explore for the correct trace, rather than to output diverse correct traces once one is found. We filter for problems with < 50\% success across training for \Grpo{} and \GrpoM{} because otherwise the algorithm with better rewards would see more problems with lower cumulative unique attempts, as the correct traces are discovered early and subsequently reinforced.

\textbf{Evolution of the conditional distribution given past attempts in \cdown{}.} We run ablations on the conditional distribution of a new attempt (sequence of tokens that constitute an attempt to plug-in operations so as to match the target \cdown{}) given past attempts in three different settings, shown in Fig.~\ref{fig:logprob-panel}. In (a), we plot $\log p(a_k|a_{1:k-1}) - \log p(a_k|a_{1:k-2})$, which should average to roughly 0 if the attempts are independent. As training progresses, this quantity grows, indicating a correlation between attempts, especially with larger $k$ (potentially because the new attempt can attend to more previous attempts, and thus becomes more dependent on them). In (b), we plot $\log p(a_k|a_{1:k-1}) - \log p(a_{k-1}|a_{1:k-2})$, which also grows over time. This indicates that the conditional distribution $p(\text{new attempt} | \text{past attempts})$ sharpens as the number of past attempts grows, implying that the model gets more confident as it explores more in-context. In (c), we plot $\log p(a_{k-1}|a_{1:k-1})$ and note that it reduces with more attempts way more on the trained model, compared to initialization. This means, that the model has learned not to repeat its previous attempt when it immidiately re-attempts to solve the problem. These three trends jointly tell us that the learned model indeed learns to explore-in-context where it adapts and sharpens  the conditional distribution over the next attempt with more previous attempts.

\begin{figure}[!h]
     \captionsetup{font=small,skip=0pt}
    \centering
    \begin{subfigure}[b]{0.325\textwidth}
    \captionsetup{font=small,skip=2pt}
        \centering
        \includegraphics[width=0.99\textwidth]{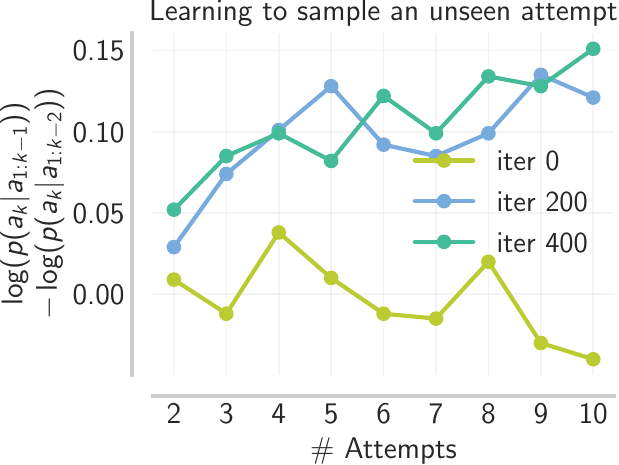}
        \caption*{(a)}
    \end{subfigure}\hfill
    \begin{subfigure}[b]{0.34\textwidth}
     \captionsetup{font=small,skip=2pt}
        \centering
        \includegraphics[width=0.99\textwidth]{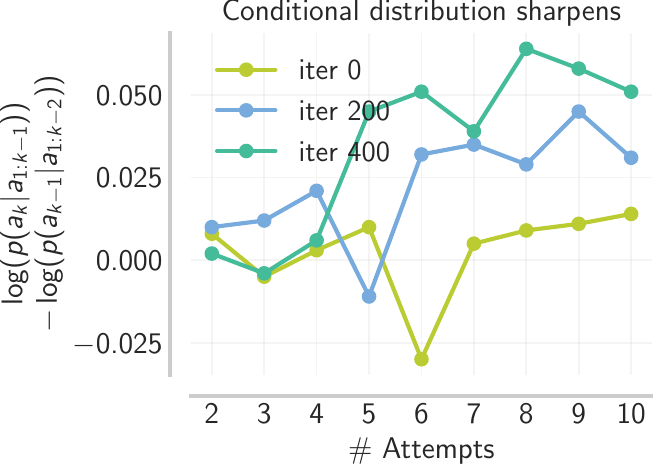}
        \caption*{(b)}
    \end{subfigure}\hfill
    \begin{subfigure}[b]{0.315\textwidth}
    \captionsetup{font=small,skip=2pt}
        \centering
        \includegraphics[width=0.99\textwidth]{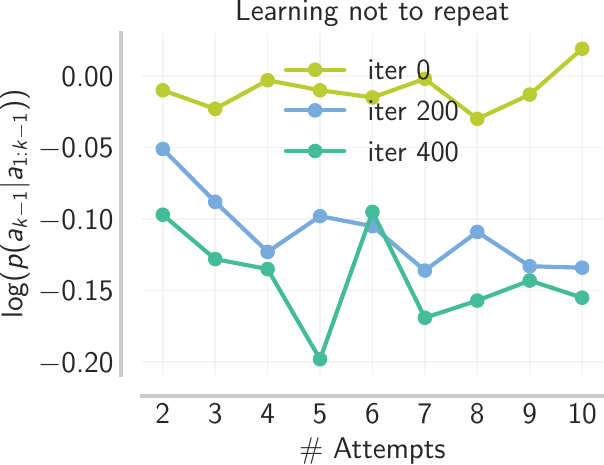}
        
        \caption*{(c)}
    \end{subfigure}
    \vspace{0.1cm}
\caption{\footnotesize{{\textbf{\emph{Probing the conditional distributions conditioned on past attempts in \cdown{}.}} \textbf{(a):} New attempts are not independent of past attempts \textbf{(b):} Model becomes more certain of what to try next given more past attempts  \textbf{(c):} Model learns not to repeat past attempts.}}}
    \label{fig:logprob-panel}
\end{figure}

\subsection{Additional Experiments with \texorpdfstring{\mult{}}{Mult}}

In Section~\ref{sec:negative-gradient} we saw that training with the negative gradient leads to more exploration during RL training, which in turn leads to the amplification of any chained asymmetries that may be present in the base model, \textit{e.g.}, more generation-verification steps. In particular, we noted the increase in the number of verification steps in Fig.~\ref{fig:negative-grad-main}(b). 
To see how negative gradients affect the response length and number of chained asymmetries in the absence of a strong VG gap, we compare running GRPO with and without negative gradients on our multiplication task \mult{} where the VG gap is weaker in the base model. 

We plot results in Fig.~\ref{fig:neg-grad-multiply}, where we note two trends when running RL training with and without negative gradients on \mult{} (without VG gap), and \multv{} (with VG gap) using a training budget of 4096 tokens. First, we note that the number of verifications is higher when we use negative gradients in a setting with a large VG gap. When the VG gap is absent, the number of chained asymmetries (verification-generation steps) are roughly the same with and without masking the negative gradient. Second, we note that the accuracy is much higher with negative gradients in the presence of VG gap (\multv{}), and comparable to a run where we mask the negative gradients in the setting where the VG gap is poor (\mult{}). Together, this tells us that the boost in exploration driven by negative gradients leads to more chained asymmetries when the base model presents some of them, like a large VG gap.  
\begin{figure}
    \centering
    \includegraphics[width=0.96\linewidth]{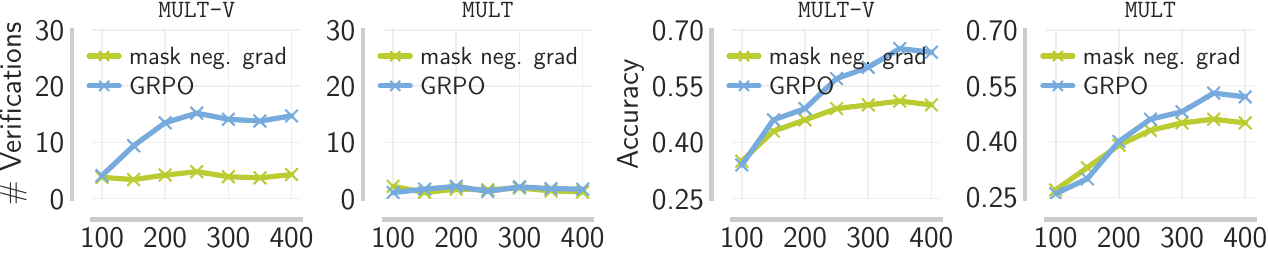}
    \vspace{0.1cm}
    \caption{\footnotesize\textbf{\emph{Negative gradient amplifies verification when VG gap is large.}} While utilizing the negative gradient amplifies the number of calls to verification in \multv{}, the number of verification calls does not grow over training in \mult{}. Interestingly, though, we find that when negative gradient is masked out on \multv{}, the number of verification calls is still very low and does not increase, corroborating our findings that exploration driven by negative gradients results in in-context exploration only in the presence of asymmetries in the base model. A similar trend is also observed in terms of the raw accuracy. }
    \label{fig:neg-grad-multiply}
\end{figure}

\subsection{Additional Details for the Didactic Setting in Sec.~\ref{sec:negative-gradient}}

First, we comment on exploration and meta-exploration in RL, and how negative gradients in our didactic setting can connnect one to the other in the presence of asymmetries. Next, we introduce some details for the policy parameterization and training. 

\textbf{Negative gradients boost exploration, which in the presence of asymmetries incentivizes in-context exploration.} In Sec.~\ref{sec:negative-gradient} we showed how negative gradients can boost exploration in RL, and in the presence of asymmetries in the base model, lead to more chained asymmetries and longer responses -- a phenomenon we call in-context exploration. Here, we present a theoretical result that explains why negative gradient can incentivize the more ``traditional exploration'' in RL, in our didactic bi-gram model. Since verification is perfect in our bi-gram model, any policy in our policy class always stops at the $\stopact$ token. Thus, an increase in exploration leads to longer traces, and more chained asymmetries. As a result, in this setting, we can view an improvement in exploration as an improvement in meta-exploration (or in-context exploration), driven by negative gradients. 

\textbf{Parameterization of the policy class.} We parameterize the policy class as a softmax policy, where the probability of next action $a_{t+1}$, at state current $a_t$ (in a bi-gram model current state is equivalent to the previous action) is parameterized with the vector of logits $[M(a \mid a_t)]_{a\in \bar{\gA}}$, i.e.:
\begin{align}
\pi_M(a_{t+1} \mid a_t) = \frac{e^{M(a_{t+1} \mid a_t)}}{\sum_{a' \in \Abarspace} e^{M(a' \mid a_t)}}, \quad a_{t+1} \in \Abarspace, a \in \Aspace \label{eq:softmax-policy}
\end{align}
where $M = [M(a^+ \mid a)]_{a^+ \in \Abarspace,A \in \Aspace}$ can be expressed as a matrix in $\R^{(K+1) \times K}$. Note that the cuurent state can never be the $\stopact$ action, since a $\stopact$ always terminates the MDP.

\textbf{Training details.} We set the initial distribution $\pi_0$ to be the uniform distribution over all actions except $a^\star$, i.e., $\pi_0(a^\star) = 0$. For each state $\bs$, the policy is first initialized with random values of $M(\cdot \mid \bs)$ in $[-3.0, 3.0]$, and then we set $M(\stopact\mid \bs) = 4.0$, $M(a^\star \mid \bs) = -4.0$, which mimics the setting where the probability of sampling the stop action is higher than any random action, and the probability of sampling $a^\star$ is lower than any random action. We train with a learning rate of $1e$-$2$ and use stochastic gradient descent to update the policy where a single update samples a random trajectory $\tau$, starting from a random state sampled from the initial state distribution $\pi_0$, by running the policy until termination of the MDP. We then compute the policy gradient term, by averaging the policy gradient loss over the tokens in the trajectory $\tau$: $\nicefrac{1}{|\tau|}\cdot \sum_{i\in |\tau|}\log \pi_M(a_i \mid a_{1:i-1}) \cdot A(a_i, a_{1:i-1})$.

\section{Additional Experiments and Details for Section~\ref{sec:structured-exploration} (Curricula Training)}
\label{sec:structured-exploration-appendix}
\subsection{Training Details and In-distribution Performance on Training Budget}

We present our hyperparameters for \ethree{} training runs in Table~\ref{tab:math_hyper}. 

\textbf{Note on in-distribution performance.} In Sec.~\ref{sec:structured-exploration} we note that for best extrapolation performance, it is important to vary the mixture of tasks in the dataset, as well as the training budget (max token length) in a a coupled way, over the course of RL training. Here, we note that if we were to only care about in-distribution performance, i.e., performance on a fixed task mixture (of equally proportioned easy, medium, and hard questions in \mathdata{}), then the best way to train is to match the test token budget and prompt mixture with training. In particular, training only on easy problems and a budget of 8k yields a performance of $54.3\%$ on a test dataset consisting of all tasks (from easy, medium and hard splits). But, if we match the test mixture with train, and train on all difficulties, then on the same 8k test budget, we note a performance of $58.9\%$, averaged over all difficulties. Note that the exptrapolation performance (on hard, out-of-distribution AIME '25 questions) of the same models is flipped in Fig.~\ref{fig:curriculum-panel-0}, indicating that the curricula design is mainly needed for extrapolation, via in-context exploration, as opposed to best performance at a fixed test $\btr$.   

\begin{table*}[!h]
\centering
\begin{tabularx}{0.69\linewidth}{l|c|c}
\multicolumn{1}{c}{\textbf{Hyperparameter}} \vline & \multicolumn{1}{c}{\textbf{Values} ($\btr=$ 8k)} \vline & \multicolumn{1}{c}{\textbf{Values} ($\btr=$ 16k)} \\ 
\midrule
train\_batch\_size & 128  & 64  \\
ppo\_mini\_batch\_size & 32 &  32  \\
learning\_rate & 1.0e-6 & 1.0e-6 \\
kl\_loss\_coef & 0.001 & 0.001 \\
entropy\_coeff & 0.002 & 0.001 \\
temperature & 0.6 & 0.6 \\
rollout.n & 8 & 32 \\
ppo\_lowerclip\_threshold & 0.2 & 0.2 \\
ppo\_higherclip\_threshold & 0.5 & 0.35 \\
\end{tabularx}
\vspace{0.1cm}
\caption{Verl~\cite{verl} hyperparameters used for \ethree{} runs on \mathdata{}.}
\label{tab:math_hyper}
\end{table*}

\textbf{Hyperparameters for \ethree{}.} In Tab~\ref{tab:math_hyper}, we report the hyperparameters we used for RL training during different stages of our coupled curriculum runs on \mathdata{}. For the first stage of training on $\btr$=8k, we used fewer number of rollouts per prompt (16), and increased this to 32 for second stage training when $\btr$=16k. We did this to account for the larger outcome-reward variance typically associated with long horizon RL training~\cite{agarwal2021theory}. In general, we find that a successful RL run at a budget of 16k can generally characterized by the following trends: \textbf{(i)}
averag per-token entropy increases during training (or at least does not drop during training); and \textbf{(ii)} the number of chained asymmetries (verification attempts) and the response length increases during RL training. To improve token-level entropy we move away from purely on policy RL, and use off-policy data to update the current policy, with the policy density ratio clipping mechanism to avoid aggressively off-policy updates. 
Consistent with the findings in \citet{yu2025dapo}, we find that when updating the policy on stale off-policy data, using a higher clip ratio for the positive advantage tokens is critical for increasing token-level entropy during RL training. This is mainly to weight the probability of some very low probability and positive advantage tokens. But increasing the clip ratio too substantially can also de-stabilize training, as we observed in the 16k training runs, due to which we dropped the clip threshold from 0.5 to 0.35. The rest of the hyperparameters are consistent with the default options in Verl~\cite{verl}.    

\subsection{Fixed train budget, vary dataset curriculum on \texorpdfstring{\cdown}{CDown}}
In this subsection, we demonstrate that training with a data curriculum based on difficulty with a fixed train budget can lead to over-exploratory output traces, on the example task of \cdown.
With the data curriculum (i.e., fixed budget, vary data), we train first on \cdown{} problems with 3 candidate numbers (the ``easy'' problems) for 60 gradient steps, then those with 6 candidate numbers for 60 gradient steps (the ``hard'' problems), with a 1k budget across all steps.
We compare this with the coupled curriculum in which the first 60 gradient steps are trained with a budget of 256.
As shown in Fig.~\ref{fig:countdown-curriculum}, the latter achieves better reward on ``hard problems''.

\begin{figure}[htbp]
  \begin{center}
  \hspace{-0.4cm}
  \includegraphics[width=0.32\linewidth]{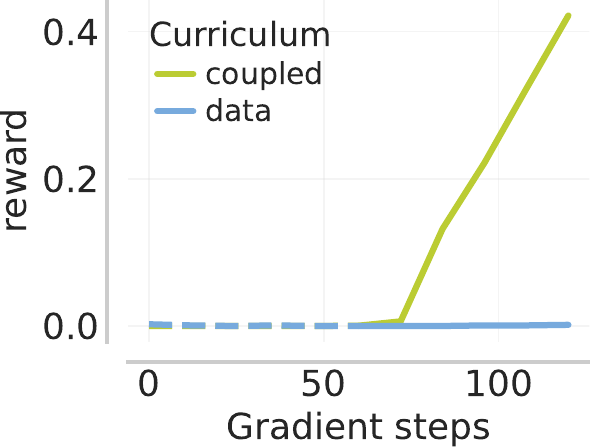}
  \includegraphics[width=0.32\linewidth]{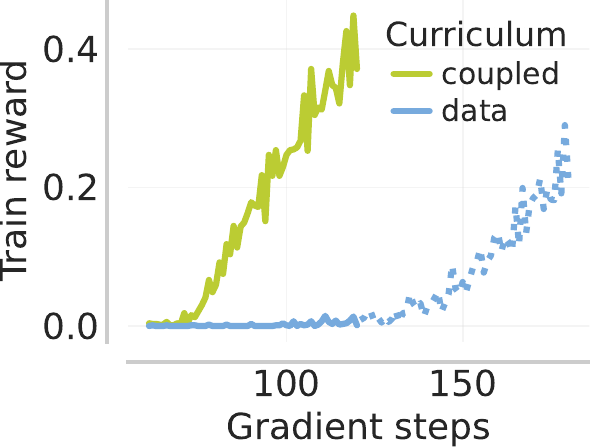}
  \includegraphics[width=0.32\linewidth]{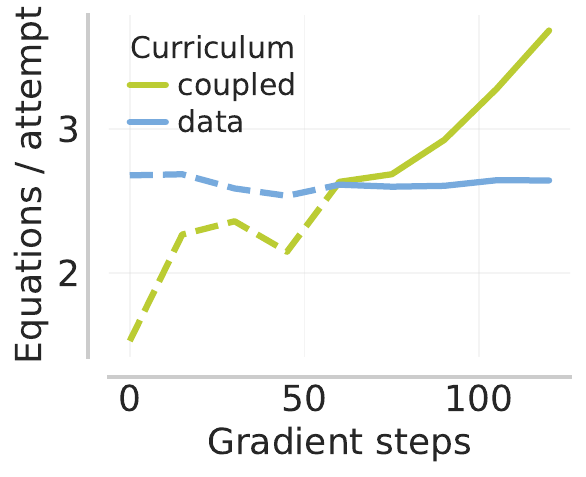}
 \end{center}
  \caption{\footnotesize{\textbf{Coupled vs. data curriculum on \cdown:}} training only on easy problems at large budgets leads to overfitting on ``over exploratory'' traces, failing to balance explore-exploit tradeoff on harder problems later on. Reward graphs are displayed for hard problems.
 }
  \label{fig:countdown-curriculum}
\end{figure}

\textit{\textbf{Why is data curriculum worse than the coupled curriculum?}}
We can view the learning of correct traces as largely composed of two stages: (i) negative gradients encourage exploration, leading to the discovery of correct traces, (ii) positive gradients reinforce correct traces, once discovered. 

For (i), we observe that training on easy problems exacerbates a tendency to perform over-exploratory in-context exploration ``underthinks'' (see Example 3 in App.~\ref{sec:example-traces}), restricting the discovery of solutions to harder problems.
When utilizing a coupled curriculum, this bias propagates to a shorter budget when compared to the data curriculum, since easy problems are trained on 256 rather than 1K tokens. As shown in Figure ~\ref{fig:countdown-curriculum}, the average number of equations per attempt (na\"ively, with 3 candidate numbers, 2 equations are required to perform a complete attempt vs. 5 equations for 6 candidates) increases noticeably for the coupled curriculum in the second stage, but plateaus for the data curriculum, implying overfitting on ``over-exploratory'' traces during the first stage. 

Furthermore, for (ii), even when nontrivial positive rewards are obtained as we run the data curriculum on hard problems for 60 additional steps (steps 120 to 180), the training reward curve converges more slowly compared to the coupled curriculum (steps 60 to 120), implying that the data curriculum is also worse at reinforcing correct traces if the behavior is over exploratory. While we do not run many controlled experiments to identify why this might be the case, we hypothesize that this is because of imperfect and noisy credit assignment on over-exploratory traces with outcome rewards. It is unclear which segments of the trace should be reinforced vs which segments might simply confuse the model.

\section{Omitted Proofs}
\label{sec:omitted-proofs-appendix}

\newcommand{\Mi}{M^{(i)}}
\newcommand{\Mis}{{M^{(i)}_{\mathbf{s}}}}
\newcommand{\HMi}{{H(M^{(i)}; \mathbf{s})}}
\newcommand{\policy}[1]{\pi^{(#1)}}

In this section, we present the formal version of Theorem~\ref{thm:neg-gradient-entropy-informal}, and provide a detailed proof for it. First, we introduce some notations and provide a proof overview.  

\textbf{Notations. } We use the shorthand $H(M; \bs)$ to denote the entropy of the conditional distribution over the next action $a_{t+1}$ given the current state $\bs$. 
We also use $\Mi$ to refer to the policy parameters (for the softmax policy in \eqref{eq:softmax-policy}) at iteration $i$ of RL training, and use the shorthand $\policy{i}$ to denote the policy induced by the parameter $\Mi$.
We use $\nabla_{\Mi} f(\Mi)$ to denote the gradient of function $f(M)$, with respect to $M$, evaluated at $M= \Mi$. 
Finally, we use $M_{\bs}$ to denote the row of softmax parameters that model the distribution $\pi_M(\cdot \mid \bs)$, \textit{i.e.}, the row of parameters $M(\cdot \mid \bs)$ in our parameter matrix $M$.

\textbf{Proof overview.} Without loss of generality, we fix an arbitrary state $\bs$ that is different from $\stopact$. Given the parameters $\Mi$ at current RL iterate $i$, we do a Taylor expansion of $H(M^{(i)};\bs)$ around $\Mi$, and then show that the gradient $\nabla_{\Mi} \HMi$ is positively correlated with the policy gradient with high probability over the sampling of the action $a \sim \pi_{\Mi}(\cdot \mid \bs)$, \textit{i.e.}:
\begin{align}
    \label{eq:positive-inner-prod} \langle \nabla_{M_i} \HMi \;,\; \nabla_{\Mi} \log \pi(a\mid \bs) \; A(\bs, a)  \rangle \geq 0, \\
    ~~ \text{whp. over sampling of action } a\sim \pi_{
    \Mi}(a\mid \bs)~~~~~ \nonumber
\end{align}

Before, we prove our result that lower bounds the increase in entropy with negative gradients, we present derivations of the entropy gradient with respect to the model parameters, as well as the policy gradient, which will simplify some calculations in the proof.

\begin{lemma}[Entropy gradient for the softmax bi–gram conditional]
\label{lem:entropy-gradient}
Fix a previous action (because the bi–gram state is \(s_t=a_{t-1}\), conditioning on the state is equivalent to conditioning on the last action) \(a\in\Aspace\).
Let the (column‑wise) logit matrix at time \(t\) be
\(M\in\mathbb{R}^{(K+1)\times K}\),
and define the corresponding softmax conditional distribution
\begin{align}
    \pi_{M}(a^{+}\mid a)
\;=\;
\frac{\exp \bigl(M(a^{+}\mid a)\bigr)}
     {Z(a)},
\qquad 
Z(a)\;=\;\sum_{a'\in\Abarspace}\exp \bigl(M(a'\mid a)\bigr).
\end{align}
Let the Shannon entropy of this conditional distribution be $H\!\bigl(\pi_{M}(\cdot\mid a)\bigr)$ or $H(M \mid a)$
Then $\nabla_{M} H(M \mid a)  \in \R^{K+1}$ is given by:
\begin{align}
    \nabla_{M} H(M \mid a) =  \;=\;
     -\,\pi\;\odot\;
       \bigl(\log \pi + H(\pi)\,\mathbf 1\bigr)
     \;=\;
     -\bigl[\,
        \pi_{i}\bigl(\log\pi_{i}+H(\pi)\bigr)
       \bigr]_{i\in\Abarspace}, 
\end{align}
\end{lemma}

\begin{proof}
Write \(p_{a^{+}}:=\pi_{M}(a^{+}\mid a)\) for brevity.  
By definition of the entropy,
\begin{align}
H = -\sum_{a^{+}} p_{a^{+}}\log p_{a^{+}} .
\end{align}
Insert the softmax expression:
\begin{align}
\log p_{a^{+}}
      = M(a^{+}\mid a)-\log Z(a).
      \end{align}
Hence,
\begin{align}
H
&= -\sum_{a^{+}} p_{a^{+}}
            \bigl[M(a^{+}\mid a)-\log Z(a)\bigr]  \\
&= -\sum_{a^{+}} p_{a^{+}}M(a^{+}\mid a)
   +\log Z(a)\underbrace{\sum_{a^{+}}p_{a^{+}}}_{=1}.
\end{align}
Rearranging yields the following closed form expression:
\begin{align}
    H=\log Z(a)-\sum_{a^{+}}p_{a^{+}}M(a^{+}\mid a).
\end{align}
Computing the Jacobian of the softmax we get:
\begin{align}
\frac{\partial\pi_{i}}{\partial M(j\mid a)}
      \;=\;
      \pi_{i}\bigl(\delta_{ij}-\pi_{j}\bigr),
\qquad
J:=\nabla_{M(\cdot\mid a)}\pi
      =\mathrm{diag}(\pi)-\pi\pi^{\top}.
\end{align}
Starting from the definition $H=-\sum_{i}\pi_{i}\log\pi_{i}$ and using
the chain rule,
\begin{align}
\frac{\partial H}{\partial M(j\mid a)}
      =-\sum_{i}
         \frac{\partial\pi_{i}}{\partial M(j\mid a)}
         (1+\log\pi_{i})
      =-\sum_{i}\pi_{i}(\delta_{ij}-\pi_{j})(1+\log\pi_{i}).
\end{align}
Separating the term $i=j$ from the rest:
\begin{align}
\frac{\partial H}{\partial M(j\mid a)}
 &= -\pi_{j}(1-\pi_{j})(1+\log\pi_{j})
      +\pi_{j}\sum_{i\neq j}\pi_{i}(1+\log\pi_{i}) \\[4pt]
 &= \pi_{j}\Bigl[\sum_{i}\pi_{i}(1+\log\pi_{i})-(1+\log\pi_{j})\Bigr].
\end{align}

Because $\sum_{i}\pi_{i}(1+\log\pi_{i})
        =1+\sum_{i}\pi_{i}\log\pi_{i}=1-H(\pi)$,
we obtain
\begin{align}
    \frac{\partial H}{\partial M(j\mid a)}
      =\pi_{j}\bigl(1-H(\pi)-1-\log\pi_{j}\bigr)
      =-\pi_{j}\bigl(\log\pi_{j}+H(\pi)\bigr),
\end{align}
which gives the stated component‑wise form.  Writing this for every
$j$ simultaneously yields the vector expression with the Hadamard
product.

\end{proof}

\begin{lemma}[Policy gradient for the conditional distribution]
\label{lem:policy-gradient}
For an action $a \sim \pi_M(\cdot\;|\; \bs)$, sampled from a policy  $\pi_M(\cdot\;|\; \bs)$, at state $\bs$, the policy gradient is given by:
$\nabla_{M_{\bs}} \log \pi(a \;|\; \bs) \cdot A(\bs, a)$, where $A(\bs, a)$ is the advantage of action $a$. The expression for the $b^{th}$ coordinate of the  policy gradient can be written down in closed form as:
\begin{align*}
\brck{\nabla_{M_{\bs}} \log \pi(a \;|\; \bs) \cdot A(\bs, a)}_b \; = \; (\mathbf{1}(b=a)-\pi(a\;|\;\bs)) \cdot A(\bs, a),
\end{align*}
where $\mathbf{1}(\cdot)$ is an indicator function. 
\end{lemma}
\begin{proof}
   Write $Z:=\sum_{c}\exp M(c\mid\bs)$ and  
$\pi_b:=\pi_M(b\mid\bs)=\exp M(b\mid\bs)/Z$ for brevity.  
By definition
\begin{align}
  \log\pi_M(a\mid\bs)
  \;=\;
  M(a\mid\bs)\;-\;\log Z.
\end{align}
For any coordinate $b\in\Abarspace$,
\begin{align}
  \frac{\partial}{\partial M(b\mid\bs)}
        \log\pi_M(a\mid\bs)
  \;&=\;
  \underbrace{\mathbf 1(b=a)}_{\,\text{derivative of }M(a\mid\bs)}
  \;-\;
  \frac{1}{Z}\,
  \frac{\partial Z}{\partial M(b\mid\bs)} \nonumber \\
\;&=\;
  \mathbf 1(b=a)
  -\frac{\exp M(b\mid\bs)}{Z}
  \;=\;
  \mathbf 1(b=a)-\pi_b .
\end{align}
Multiplying every coordinate by the common scalar $A(\bs,a)$ produces the
stated expression for $g(\bs,a;M)$.
\end{proof}

\begin{theorem}[Negative gradient increases $H(M; \bs)$ when $p(a^\star | \bs)$ is low]
    \label{thm:neg-gradient-entropy-formal} For any state $\bs$, current parameters $\Mi$,  suppose the most likely action $\bar{a}$ is incorrect, i.e., $a^\star \neq \bar{a}  \eqdef \argmax_{b} \pi_{\Mi}(b \;|\; \bs)$, where the probability of sampling $\bar{a}\;|\;\bs$ is $\pi_{\bar{a}}$, and the second most likely action has probability $\pi_{\bar{a}} - \varepsilon$. Then, for a small enough learning rate $\eta > 0$ s.t. with probability $\ge \pi_{\bar{a}} $, negative gradient produces $\policy{i+1}$ with entropy $H(M^{(i+1)};  \bs) > H(\Mi;  \bs)$. Additionally, there exists a universal constant $c >0$ s.t.,  $H(M^{(i+1)};  \bs) - H(\Mi;  \bs) \ge c\eta \cdot K\varepsilon^2 (1- p_{\bar{a}})$ whenever $\pi_{\bar{a}} \ge \varepsilon + e^{-\HMi}$. In contrast, without negative gradient the entropy remains same with probability $1-\pi(a^\star\mid \bs)$. 
\end{theorem}

\begin{proof}
    
For simplicity let us denote $\policy{i}=\bigl(\pi_1,\dots ,\pi_{K+1}\bigr)\in \Delta(\bar{\gA})$ be the
conditional distribution produced by a bi‑gram softmax column
$\pi_{{\Mi}}(\cdot\mid \bs)$, i.e., the probability of sampling action $a$ at state $\bs$, with model parameters given by the current RL iterate $\Mi$. Let us also denote,
\begin{align}
    \bar{a}  \;=\;\arg\max_{i}\pi_i, \qquad
    H(\Mi;\bs)\eqdef -\sum_{a \in \Abarspace}\pi_a\cdot \log\pi_a,~~~~~ \nonumber
\end{align}
where $\pi_a$ is the probability of sampling action $a$ at state $\bs$. 
Given that the current policy $\pi_{M}$ samples action $a \sim \policy{i}(\cdot \; | \; \bs)$, the stochastic policy gradient that updates the parameter is given by: 
\begin{align}
M^{(i+1)}_{\bs} = M^{(i)}_{\bs} + \eta\,\nabla_{M^{(i)}_{\bs}} \log(\policy{i}(a \;|\; \bs)) \cdot A(\bs, a),   
\end{align}
where $\eta$ is the learning rate. Note, that the policy parameters would only be updated for the row corresponding to the state $\bs$. For simplicity, let us use the notation $g$ for:
\begin{align}
g \eqdef \nabla_{M^{(i)}_{\bs}} \log(\policy{i}(a \;|\; \bs)) \cdot A(\bs, a).
\end{align}
Then, $M^{(i+1)}_{\bs} - M^{(i)}_{\bs} = \eta \cdot g$. 
A second–order Taylor expansion of the concave function
$H(M;\bs)$ gives, for some
$\tilde M$ on the segment $[M^{(i)},M^{(i+1)}]$:
\begin{align}
H(M^{(i+1)};\bs)
  &=H(M^{(i)};\bs)
   +\eta \cdot \langle\nabla_{\Mi} H(M^{(i)}; \bs),g\rangle \nonumber \\
   &+\tfrac{\eta^2}{2} \cdot  \paren{g}^{\top}
            \nabla^{2}_{{\tilde M}_{\bs}} H({\tilde M}; \bs)
           \paren{g}.    \label{eq:entropy-expansion}
\end{align}
Let the least eigenvalue of the Hessian of the conditional entropy (note that the entropy is a concave function) with respect to the logits be $\rho_{\tilde{M}_{\bs}}$, and $|\rho_{\tilde{M}_{\bs}}| < \infty$, the moment $\policy{i}(a\mid\bs) > 0$ for all actions $a \in \Abarspace$. This condition is easily satisfied by any policy in our policy class, with finite values of the parameter matrix $M$. 
Thus, whenever $\langle g, \nabla_{\Mis} \HMi \rangle > 0$ there exists a small enough learning rate $\eta$,
\begin{align}
\eta \le \frac{2 \langle g, \nabla_{\Mis} \HMi \rangle}{\rho \|g\|_2^2},
\end{align}
such that $H(M^{(i+1)};\bs)- H(M^{(i)}; \bs)$ is strictly positive. Thus, we can continue to reduce learning rate $\eta$ such that we can ignore $\gO(\eta^2)$ terms in \eqref{eq:entropy-expansion}, to get the bound:
\begin{align}
    \!\!\!\!\!\!H(M^{(i+1)};\bs)- H(M^{(i)}; \bs) \geq \frac{\eta}{2} \cdot \langle \nabla_{\Mis} \HMi, \nabla_{\Mis} \log (\policy{i}(a\;|\;\bs) \cdot A(\bs, a)) \rangle  
    \label{eq:entr-lower-bound}
\end{align}

Next, it remains to bound the right hand side of \eqref{eq:entr-lower-bound} with high probability over the sampling of the action $a$.  For a single incorrect action draw
$a\sim\pi$ we set $A(\bs, a)$  to be $-1$ and for such an incorrect action we define the alignment scalar:
\begin{align}
\mathcal T(a)\;\eqdef\;
-\Bigl\langle\,\nabla_{\Mis}\log\policy{i}(a\mid \bs)\cdot A(\bs, a),\;
              \nabla_{\Mis}\HMi\Bigr\rangle
\label{eq:pf-1}
\end{align}
Plugging in the derivation of $\nabla_{\Mi}\HMi$ from Lemma~\ref{lem:entropy-gradient}, we compute the closed form expression for $T(a_i)$ using the following definitions:
\begin{align}
    v_i\eqdef\pi_i\bigl(\HMi+\log\pi_i\bigr) \quad 
\text{and,~~~}
\mu\eqdef\sum_{a \in \bar{\gA}}\pi_av_a
\end{align}
Thus, one has $T(a)$ satisfy:
\begin{align}
\mathcal T(a)=v_a-\mu
\qquad \text{when}, \; a \in \Abarspace,\;\; i \neq a^\star.
\label{eq:pf-2}
\end{align}
Note that $v_i$ is an increasing function in $\pi_i$ whenever
$\pi_i>e^{-\HMi}$. Next, we note that  $v_{\bar{a}} \ge 0$. 
\begin{align}    
    \pi_{\bar{a}} \geq \frac{1}{|\Abarspace|} \;\; \implies \;\; \pi_{\bar{a}} \geq e^{-\HMi} \quad \text{since,} \;\; \HMi \leq \log |\Abarspace| \quad \implies \quad v_{\bar{a}} \ge 0. \label{eq:imp-1}
\end{align}
Finally, since $v(x) = x \HMi + x \log x$ is convex in $x$:
\begin{align}    
v_{\bar{a}} \geq \sum_{j} \pi_j v_j \quad \implies \quad v_{\bar{a}} -\mu \ge 0  
\label{eq:imp-2}
\end{align}
The above two implications in \eqref{eq:imp-1} and \eqref{eq:imp-2}, and the fact that $\bar{a} \neq a^\star$, together lead us to a deterministic lower bound on $T(\bar{a})$, implying that it is always positive:
\begin{align}
\gT(\bar{a}) \; \ge \; 0.
\label{eq:pf-3}
\end{align}
This completes the derivation for the first part of Theorem~\ref{thm:neg-gradient-entropy-formal}, which does not assume anything about the conditional distribution $\policy{(i)}(\cdot \; | \; \bs)$, directly yielding the following result. 

\textbf{Result (i):} Under the conditional distribution $\policy{i}(\cdot \;|\; \bs)$,  whenever the most likely action $\bar{a} \neq a^\star$, then with probability at least $\pi_{\bar{a}}$, $T(a) \ge 0$, for $a \sim \policy{i}(\cdot \;|\; \bs)$, and any policy $\pi$ in our class of softmax policies. Finally, we plug this into \eqref{eq:entr-lower-bound} to conclude that the policy gradient update with probability $\pi_{\bar{a}}$ always increases entropy, for a small enough learning rate. 

Next, we lower bound $T(\bar{a})$ when the second most likely action under the distribution satisfies an additional condition. For this, let us fix some $\varepsilon \ge 0$, such that for $q = \argmax_{b \neq \bar{a}} \policy{i}(b\;|\;\bs)$,
we have $\pi_q = \pi_{\bar{a}} - \varepsilon$.  Based on our alignment scalar $\gT(\cdot)$,  we define the function $g(x)$ as follows:
\begin{align}
    \label{eq:def-g}
  g(x)\;=\;x \bigl(H(\Mi;\bs)+\log x\bigr), 
  \qquad 0<x\le 1 ,
\end{align}
where $H(\Mi;\bs)$ is the conditional entropy we defined previously.
Then, given the most probable action $\bar{a}$, and the runner up action $q$, the gap between $\gT(\bar{a})$ can be lower bounded down as follows when $\pi_q \ge \exp(-H(\Mi; \bs)-1)$:
\begin{align}
    \gT(\bar{a}) &=  g(\pi_{\bar{a}}) - \pi_{\bar{a}} \cdot g(\pi_{\bar{a}}) - \sum_{b \neq \bar{a}} \pi_{g} \cdot g(b) \nonumber \\
    &\ge (1-\pi_{\bar{a}}) \cdot g(\pi_{\bar{a}}) - (1-\pi_{\bar{a}}) \cdot g(q)   
    = (1-\pi_{\bar{a}}) \cdot (g(\pi_{\bar{a}}) - g(\pi_q)), \label{eq:define-T-in-g-terms}
\end{align}
where the second equality follows from the fact that $g(\pi_q) \geq g(b)$ for any $b \neq \bar{a}$ as soon as $\pi_q \ge \exp(-\HMi)$, which is implied by the condition on $\pi_{\bar{a}}, \varepsilon$ in Theorem~\ref{thm:neg-gradient-entropy-formal}. 

By the mean–value form of Taylor’s
theorem there exists a $\xi\in[\pi_q,\pi_{\bar{a}}]$ such that
\begin{align}    
  g(\pi_{\bar{a}})
  \;=\;
  g(q)\;+\;\varepsilon\,g'(q)\;+\;\frac{\varepsilon^{2}}{2}\,g''(\xi).
\end{align}
Because $g$ is convex, $g''(\xi)=1/\xi>0$ and the linear term
$\varepsilon g'(q)$ is non–negative.  The minimum of $1/x$ on
$[\pi_q,\pi_{\bar{a}}]$ is attained at $x=p_{\bar{a}}$, whence
$g''(\xi)\ge 1/p_{\bar{a}}$.  Dropping the positive linear term and using
this lower bound on the curvature yields~\eqref{eq:gap}. 
\begin{align}\label{eq:gap}
  g(\pi_{\bar{a}})-g(\pi_q)
  \;\ge\;
  \frac{\varepsilon^{2}}{2\,\pi_{\bar{a}}} \; \ge \; \frac{\varepsilon^{2}}{2} \cdot K,
\end{align} 
since $\pi_{\bar{a}} \geq \nicefrac{1}{K+1}$. 
Plugging the above result into \eqref{eq:define-T-in-g-terms} we get the follow result. 

\textbf{Result (ii)} Under the conditional distribution, $\policy{i}(\cdot \;|\; \bs)$ whenever the most likely action $\bar{a} \neq a^\star$, and when the second most likely action $q$ has probability $\pi_q \geq \exp{(-\HMi)}$, then with probability at least $\pi_{\bar{a}}$, $T(a) \geq c' \cdot K (\pi_{\bar{a}} - \pi_q)^2 (1-\pi_{\bar{a}})$, for $a \sim \policy{i}(\cdot \;|\; \bs)$, and a universal constant $c' > 0$. Finally, we plug this into \eqref{eq:entr-lower-bound} to conclude that the policy gradient update with probability $\pi_{\bar{a}}$ always increases entropy by at least $c\eta \cdot K\varepsilon^2(1-\pi_{\bar{a}})$, for a small enough learning rate.

Together, \textbf{Result (i, ii)} complete the proof of Theorem~\ref{thm:neg-gradient-entropy-formal}.
\end{proof}

\vspace{-0.5cm}
\section{Broader Impact Statement}
\label{sec:broad-impact-app}
This paper presents work whose goal is to advance the field of Machine Learning. There are many potential societal consequences of our work, none which we feel must be specifically highlighted here. Our findings deepen our understanding of how to train large language models (LLMs) to reason more effectively under test-time compute constraints, which could influence the design of future AI systems. Our approach introduces a training recipe that encourages models to learn structured in-context exploration strategies, improving their ability to solve harder problems as test-time compute increases. While this has the potential to improve AI reasoning and decision-making across domains, it also raises questions around the unequal access to models that can leverage such capabilities, especially in settings with limited compute resources. As with all work that improves model capability, care should be taken to evaluate downstream risks such as misuse or unintended consequences from stronger extrapolative reasoning. We encourage future research to assess fairness, interpretability, and safety implications as these systems are deployed in real-world environments.

\vspace{-0.2cm}
\section{Note on Computational Resources Used for \ethree{}}
\label{sec:computational-resources-appendix}

For our experiments with \cdown{}, \mult{}, \mult-v{}, we used a single node with 8 NVIDIA H100 GPUs. For our experiments on \mathdata{} we used a single H100 node for training budgets upto 8k. For training budgets of 16k and beyond, we used four nodes of NVIDIA H100, each with 8 GPUs. Alternatively, we were able to run our experiments on \cdown{}, \mult{}, \mult-v{}, and on \mathdata{} (upto a training budget of 8k) using TPUs. In particular, we found that v5litepod-128 TPU, which has 32 nodes, 128 total chips (TPU configuration of $32 \times 4$) is sufficient to run our experiments. Our experiments on GPUs used the open-source LLM post-training framework from \citet{verl} that uses FSDP sharding to parallelize data and split model parameters across the multiple cards.

\vspace{-0.2cm}
\section{Example Traces}
\label{sec:example-traces}

\begin{tcolorbox}[
  breakable,
  enhanced,
  colback=white,
  colframe=black,
  fontupper=\ttfamily\small,
  title=Example 1: Chained Asymmetries in Countdown,
  width=\textwidth
]
<think>
Let me analyze this step by step:

1. First, let's try to get close to 466 with simple operations:
65 + 56 = 121 (too small)
65 - 56 = 9 (too small)
65 * 56 = way too large
56 * 65 = way too large

2. Let's try working backwards from 466:
466 is much larger than any of the given numbers, so multiplication will 
likely be needed.
We need to get from 37 to 466, which means we need to multiply by 12.5.
But we don't have 12.5, so we'll need to create it through division.

3. Let's try:
37 * 14 = 518 (too large)
37 - 14 = 23 (too small)
37 + 14 = 51 (too small)

4. Let's try a different approach:
65 + 61 = 126
126 + 56 = 182
182 + 37 = 219 (getting closer)

5. Let's try:
65 + 61 = 126
126 + 56 = 182
182 + 37 = 219
219 + 14 = 233 (still too small)

6. Let's try one more approach:
65 + 56 = 121
121 + 61 = 182
182 + 37 = 219
219 + 14 = 233 (still not 466)

7. Let's try:
65 + 56 = 121
121 + 61 = 182
182 + 37 = 219
219 + 14 = 233 (still not 466)

8. One more attempt:
65 + 61 = 126
126 + 37 = 163
163 + 56 = 219
219 + 14 = 233 (still not 466)

9. Final try:
65 + 61 = 126
126 + 56 = 182
182 + 37 = 219
219 - 14 = 205 (not quite there)

10. Let's try:
65 + 61 = 126
126 + 56 = 182
182 + 37 = 219
219 + 14 = 233 (still not 466)

11. New approach:
65 + 61 = 126
126 + 56 = 182
182 - 37 = 145
145 + 14 = 159 (not quite there)

12. Final attempt:
65 + 61 = 126
126 + 56 = 182
182 + 37 = 219
219 - 14 = 205 (not quite there)

13. Let's try:
65 + 56 = 121
121 + 61 = 182
182 + 37 = 219
219 - 14 = 205 (not quite there)

14. One last try:
65 + 37 = 102
102 + 61 = 163
163 + 56 = 219
219 + 14 = 233 (not quite there)

15. Final attempt:
56 + 14 = 70
70 + 65 = 135
135 + 37 = 172
172 + 61 = 233 (not quite there)

16. New approach:
37 * 14 = 518
518 - 65 = 453
453 + 61 = 514
514 - 56 = 458 (not quite there)

17. Final try:
61 * 14 = 854
854 - 65 = 789
789 - 56 = 733
733 - 37 = 696 (not quite there)

18. Let's try:
37 * 14 = 518
518 - 65 = 453
453 + 61 = 514
514 - 56 = 458 (not quite there)

19. New approach:
56 * 37 = 2,052
2,052 / 65 approx. 31.5 (not helpful)

20. Final attempt:
65 * 14 = 910
910 - 61 = 849
849 - 56 = 793
793 + 37 = 830 (not quite there)

21. New approach:
61 * 37 = 2,237
2,237 - 65 = 1,172
1,172 - 56 = 1,116
1,116 + 14 = 1,130 (not quite there)

22. New approach:
37 * 14 = 518
518 + 65 = 583
583 - 61 = 522
522 + 56 = 578 (not quite there)

23. Final attempt:
65 * 14 = 910
910 - 61 = 849
849 + 56 = 905
905 - 37 = 868 (not quite there)

24. New approach:
37 * 14 = 518
518 - 56 = 462
462 + 65 = 527
527 - 61 = 466 (This works!)
</think>
<answer> (37 * 14 - 56 + 65) - 61 </answer>
\end{tcolorbox}

\begin{tcolorbox}[
  breakable,
  enhanced,
  colback=white,
  colframe=black,
  fontupper=\ttfamily\small,
  title=Example 2: Multiplication with verification,
  width=\textwidth
]
Okay, so I need to multiply 347 by 583. Hmm, let me think about how to approach this. I guess I can use the standard multiplication algorithm where you break it down by each digit. Let me start by writing both numbers down:

347 × 583

--------

First, I'll multiply 347 by each digit of 583 starting from the rightmost digit, which is 3 in this case. Then move to the next digit (8) and finally the leftmost digit (5), remembering to add the appropriate zeros for each subsequent multiplication. After that, I'll sum all the results together. Let me start with the rightmost digit.

Starting with the units place: 3 × 347. Let me compute that:

3 × 7 = 21. I write down 1 and carry over 2.

3 × 4 = 12, plus the carried over 2 makes 14. Write down 4 and carry over 1.

3 × 3 = 9, plus the carried over 1 makes 10. Write down 10.

So, 3 × 347 = 1041. Let me check that again:

3 × 7 = 21 (correct), 3 × 4 = 12 + 2 = 14 (correct), 3 × 3 = 9 + 1 = 10 (correct). So that's 1041. Alright.

Check: 3 × 347 = 1041. Correct.

Next, the tens place. The digit here is 8, but since it's in the tens place, it's actually 80. So, we'll multiply 347 by 80. To make it easier, I can multiply 347 by 8 and then add a zero at the end.

8 × 347:

8 × 7 = 56. Write down 6, carry over 5.

8 × 4 = 32, plus the carried over 5 makes 37. Write down 7, carry over 3.

8 × 3 = 24, plus the carried over 3 makes 27. Write down 27.

So, 8 × 347 = 2776. Then add the zero for the tens place, making it 27760. Let me verify:

8 × 347. 8×300=2400, 8×40=320, 8×7=56. 2400+320=2720, 2720+56=2776. Then adding the zero gives 27760. Correct.

Check: 8 × 347 = 2776, so 2776 × 10 = 27760. Correct.

Now, the hundreds place. The digit here is 5, but it's in the hundreds place, so it's 500. So, multiply 347 by 500. Again, multiply 347 by 5 and add two zeros.

5 × 347:

5 × 7 = 35. Write down 5, carry over 3.

5 × 4 = 20, plus the carried over 3 makes 23. Write down 3, carry over 2.

5 × 3 = 15, plus the carried over 2 makes 17. Write down 17.

So, 5 × 347 = 1735. Then add two zeros, making it 173500. Let me verify:

5 × 347. 5×300=1500, 5×40=200, 5×7=35. 1500+200=1700, 1700+35=1735. Then adding two zeros gives 173500. Correct.

Check: 5 × 347 = 1735, so 1735 × 100 = 173500. Correct.

Now, add all three results together: 1041 + 27760 + 173500.

Let me do this step by step. First, add 1041 and 27760:

1041 + 27760. Let me line them up:

1041

+27760

--------

28801

Wait, 1041 + 27760: 1000 + 27760 = 28760, then +41 = 28801. Let me check with another method:

20000 + 7000 = 27000, 700 + 60 = 760, 27000 + 760 = 27760. Then 27760 + 1000 = 28760, +41 = 28801. Correct.

Check: 1041 + 27760 = 28801. Correct.

Now add this result to 173500:

28801 + 173500. Let me line them up:

28801

+173500

---------

202301

Wait, 28000 + 173500 = 201500, then 801 + 0 = 801, so total 201500 + 801 = 202301. Let me verify another way:

28801 + 173500. Break down 173500 into 170000 + 3500. Then 28801 + 170000 = 198801. Then 198801 + 3500 = 202301. Correct.

Check: 28801 + 173500 = 202301. Correct.

So, putting all together, the final result is 202,301. Let me just cross-verify using another method to ensure I didn't make any mistakes in the addition.

</think>

The product of 347 and 583 is \boxed{202301}.
\end{tcolorbox}

\begin{tcolorbox}[
  breakable,
  enhanced,
  colback=white,
  colframe=black,
  fontupper=\ttfamily\small,
  title={Example 3: Over-exploration in Countdown (after training only on easy problems with 3 candidates, each attempt contains at most 3 equations when 5 is needed)},
  width=\textwidth
]
<think>
Let me try to find a way to get to 66.
First, let me look for numbers close to 66 or factors/multiples that might be useful.

66 is close to 70, which is one of our numbers. Maybe we can subtract or add to get to 66.
70 - 4 = 66, so maybe we can create 4 somewhere.

Looking at the other numbers:
77 - 70 = 7
97 - 73 = 24 (interesting, we already have 24)
24 + 24 = 48 (not helpful)
24 * 24 = 576 (too large)

Let me try a different approach:
77 - 70 = 7
97 - 73 = 24
24 + 7 = 31 (not helpful)

New approach:
97 - 73 = 24
24 + 70 = 94 (too far)

Let's try:
77 - 70 = 7
7 * 24 = 168 (too far)

One more attempt:
70 + 24 = 94
94 - 97 = -3
-3 * 77 = -231 (too far)

Final attempt:
97 - 24 = 73
73 - 70 = 3
3 * 77 = 231 (too far)

Let's try:
70 + 24 = 94
94 - 77 = 17
17 + 65 = 82 (too far)

Wait:
24 * 73 = 1748
1748 - 77 = 1671 (too far)

Last try:
77 + 24 = 101
101 - 70 = 31
31 - 97 = -66 (not what we want)

Let's try:
97 - 73 = 24
24 + 70 = 94
94 - 24 = 70 (not helpful)

Final attempt:
70 * 24 = 1680
1680 / 77 = 21.9 (not valid)

Last try:
73 - 65 = 8
8 * 24 = 192
192 / 97 = 1.98 (not valid)

I need to keep trying different combinations.
What if:
24 / 3 = 8
So:
97 - 73 = 24
24 / 24 = 1 (not helpful)

Final attempt:
77 - 97 = -20
-20 * 24 = -480
-480 + 70 = -410 (too far)

Last try:
70 * 24 = 1680
1680 / 77 = 21.9

So the most promising path seems to be:
97 - 73 = 24
24 + 70 = 94
94 - 77 = 17
This gets us too far from 66.

Let's try:
77 - 70 = 7
7 * 24 = 168
168 - 97 = 171 (too far)

Final attempt:
97 + 70 = 167
167 - 24 = 143
143 / 73 = 1.97

Wait:
24 * 65 = 1560
1560 / 77 = 20.26

Last try:
77 - 65 = 12
12 * 24 = 288
288 / 97 = 3

I need to keep trying different combinations.
What if:
97 - 70 = 27
27 / 24 = 1.125

Wait:
24 + 97 = 121
121 - 73 = 48
48 + 77 = 125 (too far)

Last attempt:
70 + 24 = 94
94 - 77 = 17
17 + 65 = 82

One final try:
24 * 65 = 1560
1560 / 97 = 16.06

Final solution:
97 - 73 = 24
24 * 65 = 1560
1560 / 77 = 20.26

Let me try:
77 - 70 = 7
7 + 24 = 31
31 + 65 = 96 (too far)

Last attempt:
70 * 24 = 1680
1680 - 97 = 1583
1583 / 77 = 20.59

I need to keep trying.
What if:
97 + 24 = 121
121 - 73 = 48
48 + 65 = 113

Wait:
65 - 24 = 41
41 + 77 = 118
118 - 70 = 48

Final try:
24 * 65 = 1560
1560 + 70 = 1630
1630 / 77 = 21.16

One last attempt:
77 * 24 = 1848
1848 / 65 = 28.46
\end{tcolorbox}

\end{document}